%% file: main.tex
\documentclass{article}
\usepackage[utf8]{inputenc}

\title{Algorithmic Instabilities of Accelerated Gradient Descent}

\usepackage[margin=1in]{geometry}

\usepackage{amsfonts,amsmath,amssymb,amsthm,boxedminipage,color,url}
\usepackage{mathtools,thmtools}
\usepackage[round]{natbib}
\usepackage{xcolor}
\usepackage[extdef=true]{delimset}

\usepackage{tikz}
\usetikzlibrary{positioning}
\usepackage{xspace}
\usepackage{caption,subcaption}
\usepackage{enumitem}

\usepackage{hyperref}
\hypersetup{
	colorlinks=true,
	allcolors=blue,
}

\usepackage[capitalise]{cleveref}
\usepackage{crossreftools}

\declaretheoremstyle[
	    spaceabove=\topsep, 
	    spacebelow=\topsep, 
	    headfont=\normalfont\bfseries,
	    bodyfont=\normalfont\itshape,
	    notefont=\normalfont\bfseries,
	    notebraces={(}{)},
	    postheadspace=0.33em, 
	    headpunct={.},
    ]{theorem}
\declaretheorem[style=theorem]{theorem}

\declaretheoremstyle[
	    spaceabove=\topsep, 
	    spacebelow=\topsep, 
	    headfont=\normalfont\bfseries,
	    bodyfont=\normalfont,
	    notefont=\normalfont\bfseries,
	    notebraces={(}{)},
	    postheadspace=0.33em, 
	    headpunct={.},
    ]{definition}

\declaretheoremstyle[
        spaceabove=\topsep, 
        spacebelow=\topsep, 
        headfont=\normalfont\bfseries,
        bodyfont=\normalfont,
        notefont=\normalfont\bfseries,
        notebraces={}{},
        postheadspace=0.33em, 
        qed=$\blacksquare$, 
        headpunct={.},
    ]{proofstyle}
\declaretheorem[style=proofstyle,numbered=no,name=Proof]{proof}

\declaretheorem[style=theorem,sibling=theorem,name=Lemma]{lemma}
\declaretheorem[style=theorem,sibling=theorem,name=Corollary]{corollary}
\declaretheorem[style=theorem,sibling=theorem,name=Claim]{claim}

\declaretheorem[style=theorem,numbered=no,name=Theorem]{theorem*}
\declaretheorem[style=theorem,numbered=no,name=Lemma]{lemma*}
\declaretheorem[style=theorem,numbered=no,name=Corollary]{corollary*}
\declaretheorem[style=theorem,numbered=no,name=Proposition]{proposition*}
\declaretheorem[style=theorem,numbered=no,name=Claim]{claim*}
\declaretheorem[style=theorem,numbered=no,name=Fact]{fact*}
\declaretheorem[style=theorem,numbered=no,name=Observation]{observation*}
\declaretheorem[style=theorem,numbered=no,name=Conjecture]{conjecture*}

\declaretheorem[style=definition,sibling=theorem,name=Definition]{definition}

\declaretheorem[style=definition,numbered=no,name=Definition]{definition*}
\declaretheorem[style=definition,numbered=no,name=Remark]{remark*}
\declaretheorem[style=definition,numbered=no,name=Example]{example*}
\declaretheorem[style=definition,numbered=no,name=Question]{question*}

\newcommand{\agd}{\textup{NAG}\xspace}

\Crefname{claim}{Claim}{Claims}

\DeclareMathOperator*{\argmax}{arg\,max}
\DeclareMathOperator*{\argmin}{arg\,min}

\newcommand{\ceil}[1]{\lceil #1 \rceil}

\newcommand{\eqdef}{\triangleq}

\newcommand{\ifrac}[2]{{#1}/{#2}}
\newenvironment{aligni*}{\begin{math}}{\end{math}}

\newcommand{\dltx}[1]{\Delta^{x}_{#1}}
\newcommand{\dlty}[1]{\Delta^{y}_{#1}}
\newcommand{\dltm}[1]{\Delta^{m}_{#1}}
\newcommand{\dltf}[1]{\Delta^{f}_{#1}}

\newcommand{\fmp}[0]{f}

\newcommand{\xt}[0]{\tilde{x}}
\newcommand{\xh}[0]{\hat{x}}
\newcommand{\xb}[0]{\bar{x}}
\newcommand{\yt}[0]{\tilde{y}}
\newcommand{\yh}[0]{\hat{y}}
\newcommand{\yb}[0]{\bar{y}}
\newcommand{\zt}[0]{\tilde{z}}

\newcommand{\mt}[0]{\tilde{m}}

\newcommand{\ymax}[1]{y^{\mathrm{max}}_{#1}}
\newcommand{\ymin}[1]{y^{\mathrm{min}}_{#1}}

\newcommand{\R}[0]{\mathbb{R}}

\newcommand{\E}[0]{\mathbb{E}}
\newcommand{\D}[0]{\mathcal{D}}
\newcommand{\Z}[0]{\mathcal{Z}}
\newcommand{\I}[0]{\mathbf{1}}
\newcommand{\IND}[1]{\I\brk[s]*{#1}}
\newcommand{\unif}[2]{\delta_{#1,#2}^{\textup{unif}}}
\newcommand{\init}[1]{\delta_{#1}^{\textup{init}}}
\newcommand{\qcp}[0]{p}

\author{%
    Amit Attia%
    \thanks{Blavatnik School of Computer Science, Tel Aviv University; \texttt{amitattia@mail.tau.ac.il}.}
    \and
    Tomer Koren%
    \thanks{Blavatnik School of Computer Science, Tel Aviv University, and Google Research Tel Aviv; \texttt{tkoren@tauex.tau.ac.il}.}
}

\begin{document}

\maketitle

\input{paper}

\subsection*{Acknowledgements}

We thank Naman Agarwal, Yair Carmon and Roi Livni for valuable discussions.
This work was partially supported by the Israeli Science Foundation (ISF) grant no.~2549/19, by the Len Blavatnik and the Blavatnik Family foundation, and by the Yandex Initiative in Machine Learning.

\bibliographystyle{abbrvnat}
\bibliography{agd}

\appendix

\input{appendix}

\end{document}

%% file: paper.tex
\begin{abstract}
    We study the algorithmic stability of Nesterov's accelerated gradient
    method. For convex quadratic objectives, \citet{chen2018stability} proved
    that the uniform stability of the method grows quadratically with the number
    of optimization steps, and conjectured that the same is true for the general
    convex and smooth case. We disprove this conjecture and show, for two
    notions of algorithmic stability (including uniform stability), that the stability of Nesterov's accelerated method
    in fact deteriorates \emph{exponentially fast} with the number of gradient
    steps. This stands in sharp contrast to the bounds in the quadratic case, but also to known results for non-accelerated gradient methods where stability typically grows linearly with the number of steps.
\end{abstract}

\section{Introduction}

Algorithmic stability has emerged over the last two decades as a central tool for
generalization analysis of learning algorithms. While the classical approach in
generalization theory originating in the PAC learning framework appeal to uniform
convergence arguments, more recent progress on stochastic convex optimization
models, starting with the pioneering work of \citet{bousquet2002stability} and
\citet{shalev2009stochastic}, has relied on stability analysis for deriving
tight generalization results for convex risk minimizing algorithms.

Perhaps the most common form of algorithmic stability is the so called
\emph{uniform stability} \citep{bousquet2002stability}. Roughly, the
uniform stability of a learning algorithm is the worst-case change in its output
model, in terms of its loss on an arbitrary example, when replacing a single
sample in the data set used for training.
\citet{bousquet2002stability} initially used uniform
stability to argue about the generalization of empirical risk minimization with
strongly convex losses.
\citet{shalev2009stochastic} revisited this concept
and studied the stability effect of regularization on the generalization of convex
models.
Their bounds were recently improved in a variety of ways
\citep{feldman2018generalization,feldman2019high,bousquet2020sharper} and their
approach has been influential in a variety of settings (e.g.,
\citealp{koren2015fast,gonen2017fast,charles2018stability}).
In fact, to this day, algorithmic stability is essentially the only general
approach for obtaining tight (dimension free) generalization bounds for convex
optimization algorithms applied to the empirical risk (see
\citealp{shalev2009stochastic,feldman2016generalization}).

Significant focus has been put recently on studying the stability properties of
iterative optimization algorithms. \citet{hardt2016train} considered stochastic
gradient descent (SGD) and gave the first bounds on its uniform stability for
a convex and smooth loss function, that grow linearly with the
number of optimization steps. As observed by \citet{feldman2018generalization}
and \citet{chen2018stability}, their arguments also apply with minor
modifications to full-batch gradient descent (GD). \citet{bassily2020stability}
exhibited a significant gap in stability between the smooth and non-smooth
cases, showing that non-smooth GD and SGD are inherently less stable than their
smooth counterparts. Even further, algorithmic stability has also been used as
an analysis technique in stochastic mini-batched iterative optimization (e.g.,
\citealp{wang2017memory,agarwalAKTZ20}), and has been proved crucial to the
design and analysis of differentially private optimization algorithms
\citep{wu2017bolt,bassily2019private,feldman2020private}, both of which focusing
primarily on smooth optimization.

Having identified smoothness as key to algorithmic stability of iterative
optimization methods, the following fundamental question emerges: how stable are
\emph{optimal methods} for smooth convex optimization? 
In particular, what is the algorithmic stability of the celebrated Nesterov
accelerated gradient (\agd) method~\citep{nesterov1983method}---a cornerstone 
of optimal methods in convex optimization? Besides being a basic and natural 
question in its own right, 
its resolution could have important implications to the design and analysis of
optimization algorithms, as well as serve to deepen our understanding of the
generalization properties of iterative gradient methods. 
\citet{chen2018stability} addressed this question in the case of convex
\emph{quadratic} objectives and derived bounds on the uniform stability of
\agd that grow quadratically with the number of gradient steps (as opposed to
the linear growth known for GD). They conjectured that similar bounds hold
true more broadly, but fell short of proving this for general convex and
smooth objectives. Our work is aimed at filling this gap.

\subsection{Our Results}

We establish tight algorithmic stability bounds for the Nesterov accelerated
gradient method (\agd). We show that, somewhat surprisingly, the uniform
stability of \agd grows \emph{exponentially fast} with the number of steps in
the general convex and smooth setting. Namely, the uniform stability of
$T$-steps \agd with respect to a dataset of $n$ examples is in general
$\exp(\Omega(T))/n$, and in particular, after merely $T = O(\log{n})$ steps the
stability becomes the trivial $\Omega(1)$.
This result demonstrates a sharp contrast between the stability of \agd in the quadratic case and in the general convex, and disproves the conjecture of \citet{chen2018stability} that the uniform stability of \agd in the general convex setting is $O\brk{\ifrac{T^2}{n}}$, as in the case of a quadratic objective.

Our results in fact apply to a simpler notion of stability---one that is arguably more fundamental in the context of iterative optimization methods---which we term \emph{initialization stability}. 
The initialization stability of an algorithm $A$ (formally defined in \cref{sec:prelims} below) measures the sensitivity of $A$'s output to an $\epsilon$-perturbation in its initialization point.
For this notion, we demonstrate a construction of a smooth and convex objective function such that, for sufficiently small $\epsilon$, the stability of $T$-steps \agd is lower bounded by $\exp(\Omega(T))\epsilon$.
Here again, we exhibit a dramatic gap between the quadratic and general convex cases: for quadratic objectives, we show that the initialization stability of \agd is upper bounded by $O(T\epsilon)$.

For completeness, we also prove initialization stability upper bounds in a few relevant convex optimization settings: for GD, we analyze both the smooth and non-smooth cases; for \agd, we give bounds for quadratic objectives as well as for general smooth ones.
\cref{tbl:comp_methods_v1} summarizes the stability bounds we establish compared to existing bounds in the literature.
Note in particular the remarkable exponential gap between the stability bounds for GD and \agd in the general smooth case, with respect to both stability definitions.
Stability lower bounds for \agd are discussed in \cref{sec:initialization,sec:uniform}; initialization stability upper bounds for the various settings are given in \cref{sec:initi_stab}, and additional uniform stability bounds are detailed in \cref{sec:additional-unif}.

\begin{table}[ht]
\centering
\begin{tabular}{llccc}
\hline
{\sc Method}              & {\sc Setting}       & {\sc Init.~Stability} & {\sc Unif.~Stability} & {\sc Reference}      \\ \hline
GD & convex, smooth & $\boldsymbol{\Theta(\epsilon)}$ & $\Theta\brk{\ifrac{T}{n}}$  & \citet{hardt2016train}
\\
GD & convex, non-smooth & $\boldsymbol{\Theta(\epsilon+\eta\smash{\sqrt{T}})}$            & $\Theta\brk{\eta\smash{\sqrt{T}}+\ifrac{\eta T}{n}}$ & \citet{bassily2020stability}
\\
\agd & convex, quadratic & $\boldsymbol{O(T\epsilon)}$ & $\Theta\brk{\ifrac{T^2}{n}}$ & \citet{chen2018stability}
\\
\agd & convex, smooth & $\boldsymbol{\exp(\Theta(T))\epsilon}$ & $\boldsymbol{\ifrac{\exp(\Theta(T))}{n}}$ & (this paper)
\\ 
\hline
\\
\end{tabular}
\caption{\label{tbl:comp_methods_v1} Stability bounds introduced in this work ({\bfseries in bold}) compared to existing bounds. For simplicity, all bounds in the smooth case are for $\eta = \Theta(1/\beta)$.
The lower bounds for \agd are presented here in a simplified form and the actual bounds exhibit a fluctuation in the increase of stability; see also~\cref{fig:fig2} and the precise results in~\cref{sec:initialization,sec:uniform}.}
\end{table}

Finally, we remark that our focus here is on the general convex (and smooth)
case, and we do not provide formal results for the strongly convex case.
However, we argue that stability analysis in the latter case is not as
compelling as in the general case. Indeed, a strongly convex objective admits a
unique minimum, and so \agd will converge to an $\epsilon$-neighborhood of this
minimum in $O(\log(1/\epsilon))$ steps from any initialization, at which point
its stability becomes $O(\epsilon)$; thus, with strong convexity perturbations
in initialization get quickly washed away as the algorithm rapidly converges to
the unique optimum. (A similar reasoning also applies to uniform stability with strongly convex losses.)

\subsection{Overview of Main Ideas and Techniques}

We now provide some intuition to our constructions and highlight some of the key ideas leading to our results. We start by revisiting the analysis of the quadratic case which is simpler and better understood.

\paragraph{Why \agd is stable for quadratics:}

Consider a quadratic function $f$ with Hessian matrix $H \succeq 0$.
For analyzing the initialization stability of \agd, let us consider two runs of the method initialized at $x_0, \xt_0$ respectively, and let $(x_t,y_t)$, $(\xt_t,\yt_t)$ denote the corresponding \agd iterates at step $t$. 
Further, let us denote by $\dltx{t} \eqdef x_t-\xt_t$ the difference between the two sequences of iterates.
Using the update rule of \agd (see \cref{eq:agd-xt,eq:agd-yt} below) and the fact that for a quadratic $f$, differences between gradients can be expressed as $\nabla f(x) - \nabla f(x') = H (x-x')$ for any $x,x' \in \R^d$, it is straightforward to show that the distance $\dltx{t}$ evolves according to 
\begin{align*}
    \dltx{t+1}
    &= (I-\eta H) \brk!{ (1+\gamma_t) \dltx{t} - \gamma_{t} \dltx{t-1} }
    .
\end{align*}
This recursion can be naturally put in matrix form, leading to:
\begin{align*}
    \begin{pmatrix}
        \dltx{t+1} \\
        \dltx{t}
    \end{pmatrix} 
    &=
    \prod_{k=1}^t
    \begin{pmatrix}
        (1+\gamma_k)A & -\gamma_k A \\
        I & 0
    \end{pmatrix}
    \begin{pmatrix}
        \dltx{1} \\
        \dltx{0}
    \end{pmatrix}
    ,
\end{align*}
where here $A = I-\eta H$.
Thus, for a quadratic $f$, bounding the divergence $\norm{\dltx{t}}$ between the two \agd sequences reduces to controlling the operator norm of the matrix product above, namely
\begin{align*}
    \norm*{
    \prod_{k=1}^t
    \begin{pmatrix}
        (1+\gamma_k)A & -\gamma_k A \\
        I & 0
    \end{pmatrix}
    }
    \,.
\end{align*}
Remarkably, it can be shown that this norm is $O(t)$ for any $0 \preceq A \preceq I$ and any choice of $-1 \leq \gamma_1,\ldots,\gamma_t \leq 1$.
(This can be seen by writing the Schur decomposition of the involved matrices, as we show in \cref{sec:agd_quad}.%
\footnote{\citet{chen2018stability} give an alternative argument based on Chebyshev polynomials.})
As a consequence, the initialization stability of \agd for a quadratic objective $f$ is shown to grow only linearly with the number of steps $t$.

\paragraph{What breaks down in the general convex case:}

For a general convex (twice-differentiable and smooth) $f$, the Hessian matrix
is of course no longer fixed across the execution. Assuming for simplicity the
one-dimensional case, similar arguments show that the relevant operator norm is
of the form 
\begin{align*}
    \norm*{
    \prod_{k=1}^t
    \begin{pmatrix}
        (1+\gamma_k)A_k & -\gamma_k A_k \\
        I & 0
    \end{pmatrix}
    }
    \,,
\end{align*}
where $0 \leq A_1,\ldots,A_t \leq 1$ are related to Hessians of $f$ taken at
suitable points along the optimization trajectory. However, if $A_k$ are allowed
to vary arbitrarily between steps, the matrix product above might explode
exponentially fast, even in the one-dimensional case! Indeed, fix $\gamma_k =
0.9$ for all $k$, and set $A_k = 0$ whenever $k \bmod 3 = 0$ and $A_k = 1$
otherwise; then using simple linear algebra the operator norm of interest can be
shown to satisfy
\begin{align*}
    \norm3{\brk3{
    \begin{pmatrix}
        0 & 0 \\
        1 & 0
    \end{pmatrix}
    \begin{pmatrix}
        1.9 & -0.9 \\
        1 & 0
    \end{pmatrix}
    \begin{pmatrix}
        1.9 & -0.9 \\
        1 & 0
    \end{pmatrix}
    }^{t/3}}
    =
    \norm3{
    \begin{pmatrix}
        0 & 0 \\
        2.71 & -1.71
    \end{pmatrix}
    ^{t/3}}
    \geq
    1.15^t
    .
\end{align*}

\paragraph{How a hard function should look like:}

The exponential blowup we exhibited above hinged on a worst-case sequence
$A_1,\ldots,A_t$ that varies significantly between consecutive steps. It remains
unclear, however, what does this imply for the actual optimization setup we care
about, and whether such a sequence can be realized by Hessians of a convex and
smooth function $f$. Our main results essentially answer the latter question on
the affirmative and build upon a construction of such a function $f$ that
directly imitates such a bad sequence. 

\begin{figure}[ht]
    \centering
    \includegraphics[width=0.75\linewidth,trim=0.0cm 8.5cm 0.5cm 3cm,clip]{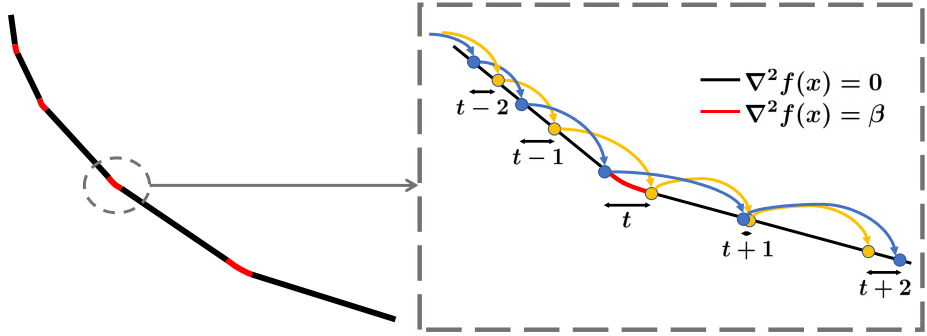}  
    \caption{A function (left) constructed of four instantiations of our ``gadget'' (right) at increasing sizes. During an interval with zero Hessian, the trajectory with the larger momentum gains distance. When reaching an interval with maximal Hessian (depicted here as iteration $t$), the ``slow'' trajectory experiences a larger gradient which gives it larger momentum and makes it become the ``faster'' one.}
    \label{fig:fig1}
\end{figure}

Concretely, we generate a hard function inductively based on a running execution of \agd, where in each step we amend the construction with a ``gadget'' function having a piecewise-constant Hessian (that equals either $0$ or the maximal $\beta$); see \cref{fig:fig1} for an illustration of this construct.
The interval pieces are carefully chosen based on the \agd iterates computed so far in a way that a slightly perturbed execution would traverse through intervals with an appropriate pattern of Hessians that induces a behaviour similar to the one exhibited by the matrix products above, leading to an exponential blowup in the stability terms.
\cref{fig:fig2} shows a simulation of the divergence between the two trajectories of \agd on the objective function we construct, illustrating 
how the divergence fluctuates between positive and negative values, with its absolute value growing exponentially with time.
More technical details on this construction can be found in \cref{sec:initialization}.

\begin{figure}[ht]
     \centering
     \begin{subfigure}[b]{0.49\linewidth}
         \centering
         \includegraphics[width=\textwidth]{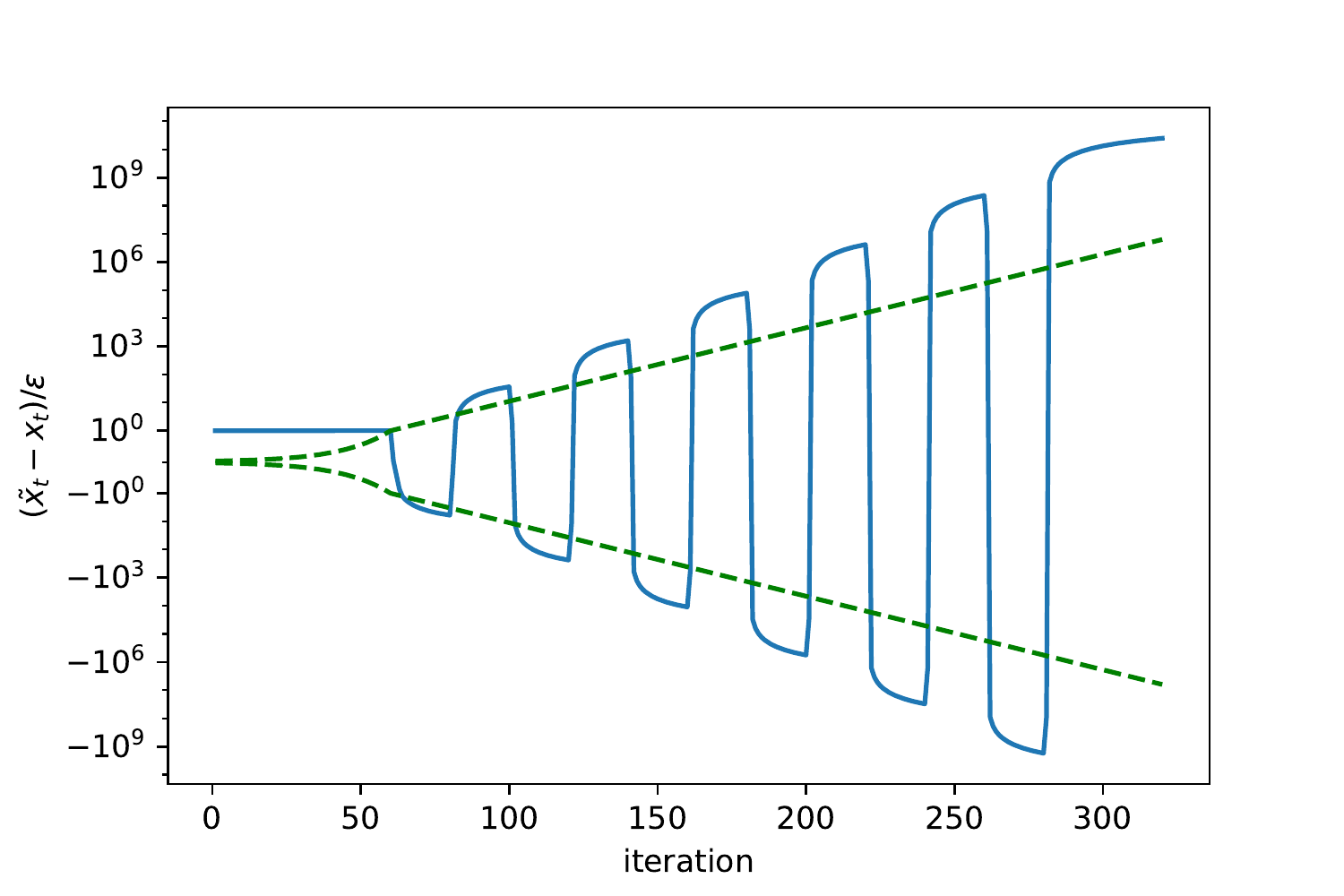}
         \caption{\small construction for $\eta=\ifrac{1}{(2\beta)}$;}
         \label{fig:eta_0.5}
     \end{subfigure}
     \begin{subfigure}[b]{0.49\linewidth}
         \centering
         \includegraphics[width=\textwidth]{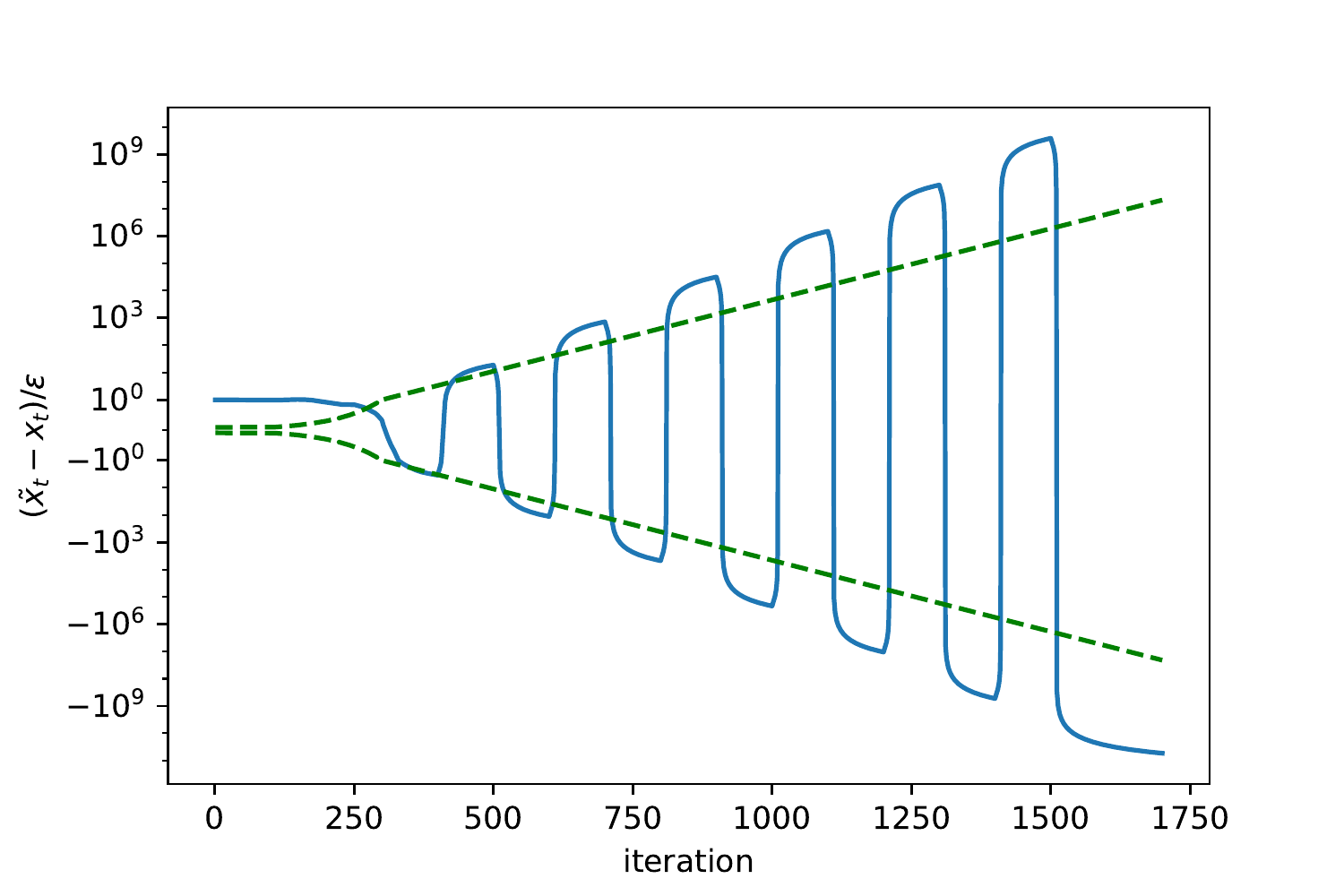}
          \caption{\small construction for $\eta=\ifrac{1}{(10\beta)}$.}
         \label{fig:eta_0.1}
     \end{subfigure}
    \caption{Divergence between trajectories (in log-scale) along the optimization process for different values of $\eta$. At steps $t = \Theta(\ifrac{i}{\eta\beta})$ ($i=1,2,\ldots$) \agd experiences an exponential growth in the divergence, as it reaches an interval with maximal Hessian. The dashed lines depict our theoretical exponential lower bound.}
    \label{fig:fig2}
\end{figure}

\paragraph{From initialization stability to uniform stability:}

Finally, we employ a simple reduction to translate our results regarding initialization stability to relate to uniform stability in the context of empirical risk optimization, where one uses (full-batch) \agd to minimize the (convex, smooth) empirical risk induced by a sample $S$ of $n$ examples.
Concretely, we show that by replacing a single example in $S$, we can arrive at a scenario where after one step of \agd on the original and modified samples the respective iterates are $\epsilon=\Theta(1/n)$ away from each other, whereas in the remaining steps both empirical risks simulate our bad function from before. Thus, we again end up with an exponential increase in the divergence between the two executions, that leads to a similar increase in the algorithmic (uniform) stability: the latter becomes as large as $\Omega(1)$ after merely $T = O(\log{n})$ steps of full-batch \agd.
The formal details of this reduction can be found in \cref{sec:uniform}.

\subsection{Discussion and Additional Related Work}

It is interesting to contrast our results with what is known for the closely
related heavy ball method~\citep{polyak1964some}. Classic results show that
while for convex quadratic objectives the (properly tuned) heavy ball method
attains the accelerated $O(1/T^2)$ convergence rate, for general convex and
smooth functions it might even fail to converge at all (see
\citealp{lessard2016analysis}). More specifically, it is known that there exists
objectives for which heavy ball assumes a cyclic trajectory that never gets
close to the optimum; it is then not hard to turn such a construction to an
instability result for heavy ball, as a slight perturbation in the cyclic
pattern can be shown to make the method converge to optimum. 

Also related to our work is \citet{devolder2014first}, that analyzed GD and \agd
with \emph{inexact first-order information}, namely, in a setting where each
gradient update is contaminated with a bounded yet arbitrary perturbation.
Interestingly, they showed that in contrast to GD, \agd suffers from an
accumulation of errors---which appears analogous to the linear increase in
initialization stability the latter experiences in the quadratic case. At the same
time, in the general convex case their results might seem to be at odds with
ours as we show that even a single perturbation at initialization suffices for
extreme instabilities. However, note that they analyze the impact of
perturbations on the \emph{convergence rate} of \agd (in terms of objective value),
whereas algorithmic stability is concerned with their effect on the actual
\emph{iterates}: specifically, initialization stability captures to what extent the
iterates of the algorithm might stray away from their original positions as a result of a small perturbation in the initialization point.

Our work leaves a few intriguing open problems for future investigation. Most
importantly, it remains unclear whether there exists a different accelerated
method (one with the optimal $O(1/T^2)$ rate for smooth and convex objectives)
that is also $\mathrm{poly}(T)$-stable. \citet{bubeck2015geometric} suggested a
geometric alternative to \agd that comes to mind, and it could be interesting to
check whether this method or a variant thereof is more stable than \agd. 
Another open question is to resolve the gap between our stability lower and upper bounds for \agd in the regime $\eta \ll 1/\beta$: while our lower bounds have an exponential dependence on $\eta$, the upper bounds do not.
Finally, it could be interesting to determine whether the $O(T\epsilon)$ initialization stability bound we
have for \agd in the quadratic case is tight (the corresponding uniform
stability result is actually tight even for linear losses, but this may not be
the case for initialization stability).

\section{Preliminaries}
\label{sec:prelims}

In this work we are interested in optimization of convex and smooth functions
over the $d$-dimensional Euclidean space $\R^d$. A function $f$ is said to be
$\beta$-smooth (for $\beta>0$) if its gradient is $\beta$-Lipschitz, namely, if
for all $u,v \in \R^d$ it holds that $\norm{\nabla f(u) - \nabla f(v)} \leq
\beta \norm{u-v}$.

\subsection{Nesterov Accelerated Gradient Method}

The Nesterov Accelerated Gradient (\agd) method \citep{nesterov1983method} we
consider in this paper takes the following form. Starting with $x_0$ and $y_0 =
x_0$, it iterates for $t=1,2,\ldots$:
\begin{align}
    \label{eq:agd-xt}
    x_{t} &= y_{t-1} - \eta \nabla f(y_{t-1});
    \\
    \label{eq:agd-yt}
    y_{t} &= x_{t} + \gamma_{t} (x_{t}-x_{t-1}),
\end{align}
where $\gamma_{t} = \smash{\frac{t-1}{t+2}}$ and $\eta>0$ is a step-size parameter. 
For a $\beta$-smooth convex objective $f$ and $0 < \eta \leq 1/\beta$, this method exhibits the convergence rate $O(1/\eta T^2)$;
for $\eta = 1/\beta$, this gives the optimal convergence rate for the class of $\beta$-smooth convex functions (see \cite{nesterov2003introductory}).
We remark that while \agd appears in several other forms in the literature, many of these are in fact equivalent to the one given in \cref{eq:agd-xt,eq:agd-yt}. For more details, see \cref{appendix:variants}.

Throughout, we use the notation $\agd(f,x_0,t,\eta)$ to refer to the iterates $(x_t,y_t)$ at step $t$ of \agd on $f$ initialized at $x_0$ with step size $\eta$. 
We sometimes drop the step size argument and use the shorter notation $\agd(f,x_0,t)$ when $\eta$ is clear from the context.

We will use the following definitions and relations throughout.
We introduce the following notation for the momentum term of \agd, for all $t > 0$:
\begin{align}
    m_t &\eqdef \gamma_{t}(x_{t}-x_{t-1}). \label{eq:defm}
\end{align}
Using this notation, we have that
\begin{align}
    x_{t} &= y_{t-1} - \eta \nabla f(y_{t-1}), \label{eq:xstep} \\
    y_{t} &= x_{t} + m_{t} = y_{t-1} -\eta\nabla f(y_{t-1}) + m_t, \label{eq:ystep} \\
    m_{t} &= \gamma_{t}(m_{t-1}-\eta\nabla f(y_{t-1})). \label{eq:mstep}
\end{align}
Here, \cref{eq:mstep} follows from \cref{eq:defm,eq:xstep,eq:ystep} via
\begin{align*}
    m_t 
    &= \gamma_t(x_t-x_{t-1}) \tag{\cref{eq:defm}} \\
    &= \gamma_t(y_{t-1}-\eta\nabla f(y_{t-1})-x_{t-1}) \tag{\cref{eq:xstep}} \\
    &= \gamma_t(x_{t-1}+m_{t-1}-\eta\nabla f(y_{t-1})-x_{t-1}) \tag{\cref{eq:ystep}} \\
    &= \gamma_{t}(m_{t-1}-\eta\nabla f(y_{t-1}))
    .
\end{align*}

\subsection{Algorithmic Stability}

We consider two forms of algorithmic stability. The first is the well-known \emph{uniform stability}~\citep{bousquet2002stability}, while the second is \emph{initialization stability} which we define here.

\paragraph{Uniform stability.}

Consider the following general setting of supervised learning. 
There is a sample space $\Z$ of examples and an unknown distribution $\D$ over $\Z$. We receive a training set $S=(z_1,\dots,z_n)$ of $n$ samples drawn i.i.d.~from $\D$.
The goal is finding a model $w$ with a small \emph{population risk}:
\begin{align*}
    R(w) 
    \eqdef 
    \E_{z \sim \D} [ \ell(w;z) ]
    ,
\end{align*}
where $\ell(w;z)$ is the loss of the model described by $w$ on an example $z$.
However, as we cannot evaluate the population risk directly, learning algorithms will be applied on the \emph{empirical risk} with respect to the sample $S$, given by
\begin{align*}
    R_S(w) 
    \eqdef 
    \frac{1}{n} \sum_{i=1}^{n} \ell(w;z_i)
    .
\end{align*}
In this paper, our algorithm of interest in this context is \emph{full-batch} \agd, namely, \agd applied to the empirical risk $R_S$.
We use the following notion of \textit{uniform stability}.%
\footnote{We give here a definition suitable for deterministic algorithms, which
suffices for the context of this paper. Similar definitions exist for
randomized algorithms; see for example
\cite{hardt2016train,feldman2018generalization}.}

\begin{definition}[uniform stability] \label{def:uni_stab}
Algorithm $A$ is $\epsilon$-uniformly stable if for all $S, S' \in \Z^n$ such that $S, S'$ differ in at most one example, the corresponding outputs $A(S)$ and $A(S')$ satisfy
\begin{align*}
    \sup_{z \in \Z} ~ \abs{\ell(A(S);z)-\ell(A(S');z)} 
    \leq 
    \epsilon.
\end{align*}
We use $\unif{A}{\ell}(n)$ to denote the infimum over all $\epsilon>0$ for which this inequality holds.
\end{definition}

\paragraph{Initialization stability.}

A second notion of algorithmic stability that we define and discuss in this paper, natural in the context of iterative optimization methods, pertains to the stability of the optimization algorithm with respect to its initialization point. \emph{Initialization stability} measures the sensitivity of the algorithm's output to a small perturbation in its initial point; formally,

\begin{definition}[initialization stability]
Let $A$ be an algorithm that when initialized at a point $x \in \R^d$, produces $A(x) \in \R^d$ as output. Then for $\epsilon>0$, the initialization stability of $A$ at $x_0 \in \R^d$ is given as
\begin{align*}
    \init{\text{A}}(x_0,\epsilon) 
    = 
    \sup\{ \norm{A(\xt_0)-A(x_0)} ~:~ \xt_0 \in \R^d,~ \norm{\xt_0-x_0} \leq \epsilon \}
    .
\end{align*}
\end{definition}

\section{Initialization Stability of \agd} \label{sec:initialization}

In this section we prove our first main result, regarding the initialization stability of \agd:

\begin{theorem} \label{thm:main_init_stab}
Let $\epsilon,G,\beta>0$ and $0<\eta \leq \ifrac{1}{\beta}$. Consider two initialization points $x_0 = 0, \tilde{x}_0 = \epsilon$.
Then, there exists a convex, $\beta$-smooth, $G$-Lipschitz function $f$ that attains a minimum over $\R$, and universal constants $c_1,c_2>0$, such that the sequences $(x_t,y_t)=\agd(f,x_0,t,\eta)$ and $(\xt_t,\yt_t)=\agd(f,\tilde x_0,t,\eta)$ satisfy
    $$
        \init{\agd_t}(x_0,\epsilon)
        \geq
        \abs{x_{t}-\xt_{t}} 
        \geq 
        \min\brk[c]!{\tfrac{G}{3\beta}, c_2 e^{c_1 \eta\beta t}\epsilon},
        \qquad
        \forall ~ t \in \brk[c]!{\ceil{\tfrac{10}{\eta\beta}}(i+2) ~:~ i = 1,2, \ldots}
        .
    $$
Furthermore, for all $t > \ceil{\frac{10}{\eta\beta}} \brk!{ \ln\frac{3G}{2\beta\epsilon}+3}$ it holds that
$
    \init{\agd_t}(x_0,\epsilon)
    \geq 
    \frac{G}{3\beta}
    .
$
\end{theorem}

In words, the theorem establishes an exponential blowup in the distance between the two trajectories $x_{t}$ and $\xt_{t}$ during the initial $O(1/\epsilon)$ steps, after which the (lower bound on the) distance reaches a constant and stops increasing.
Notice that in the blowup phase, an increase in distance happens roughly every  $\eta\beta$ steps; indeed, the actual behaviour of \agd on the function we construct exhibit fluctuations in the difference $x_{t}-\xt_{t}$, as illustrated in \cref{fig:fig2}.
We remark that a similar bound holds also for the $y_t$ sequence produced by \agd.

\paragraph{Construction.}

Throughout this section, we will assume without loss of generality that $0 < \epsilon < \ifrac{G}{2\beta}$. (When $\epsilon \geq \ifrac{G}{2\beta}$ our result holds simply for a constant function.)
To lower bound the initialization stability and prove the theorem, we will rely on the following construction of functions $f_0,f_1,\ldots : \R \to \R$. Let the parameters $G,\beta,\eta,\epsilon > 0$ be given, and for all $i \geq 0$ define $n_i \eqdef \ceil{10/\eta\beta}(i+2)$.
The construction proceeds as follows:
\begin{enumerate}[label=(\roman*),nosep]
    \item
    Let $f_0(x) \eqdef -Gx$;
    
    \item For $i \geq 1$:
    \begin{itemize}
        \item
        Let $(x_{n_i},y_{n_i})=\agd(f_{i-1},0,n_i,\eta)$ and $(\xt_{n_i},\yt_{n_i})=\agd(f_{i-1},\epsilon,n_i,\eta)$;
        
        \item
        Define $f_{i} : \R \to \R$ as follows:
        \begin{align*}
            f_{i}(x) &\eqdef -Gx + \beta \int_{-\infty}^x \int_{-\infty}^{y} \IND{\exists ~ j \leq i \text{ s.t. } z \in [\ymin{n_j},\ymax{n_j}]} dz dy,
        \end{align*}
        where $\ymin{n_i} = \min\{y_{n_i},\yt_{n_i}\}$, $\ymax{n_i} = \max\{y_{n_i},\yt_{n_i}\}$;
    \end{itemize}

    \item
    Let $M = \sup\brk[c]*{i \geq 0 ~:~ \max_x \nabla f_{j}(x) < -\frac12 G, ~~ \forall ~ 0 \leq j \leq i }$.
\end{enumerate}

Note that the above recursion defines an infinite sequence of functions $f_0,f_1,f_2,\ldots : \R \to \R$. Ultimately, we will be interested in the functions $\brk[c]{f_i}_{i \leq M}$ which we will analyze in order to prove the instability result.
Further, note that $\max_x \nabla f_0(x) < -\frac12 G$, thus $M$ itself is well-defined, possibly $\infty$.
The functions constructed above are not lower-bounded (and thus do not admit a minimum). Below we define a lower-bounded adaptation of $f_M$ (assuming $1 \leq M <\infty$, which is proved in \cref{lemma:finite_m} later on).
Our modified version of $f_M$, termed $\fmp$ is defined by a quadratic continuation of $f_M$ right of $\qcp \eqdef \ymax{n_{M}}$, up to a plateau. This construction is defined formally as:
\begin{align*}
    \fmp(x) &\eqdef
    \begin{cases}
    f_M(x) &\qquad x \leq\qcp; \\
    f_M(\qcp)+\nabla f_M(\qcp)(x-\qcp)+\frac{\beta}{2} (x-\qcp)^2 &\qquad\qcp < x \leq \qcp - \frac{1}{\beta} \nabla f_M(\qcp); \\
    f_M(\qcp)-\frac{1}{2\beta} \nabla f_M(\qcp)^2 &\qquad \text{otherwise}.
    \end{cases}
\end{align*}

\paragraph{Analysis.}

We start by stating a few lemmas we will use in the proof of our main theorem.
Our focus is on the functions $f_i$ for $0 \leq i \leq M$, deferring the analysis of $\fmp$ to after we establish that $M$ is finite.
First, we show that the functions we constructed are indeed convex, smooth and Lipschitz.

\begin{lemma} \label{lemma:f_i}
For all $0 \leq i \leq M$, the function $f_i$ is convex, $\beta$-smooth and $G$-Lipschitz.
\end{lemma}

\begin{proof}
The second derivative of $f_i$ is
\begin{align*}
    \nabla^2 f_i(x) = \beta\cdot\IND{\exists ~ j \leq i ~\text{s.t.}~ z \in [\ymin{n_j},\ymax{n_j}]} \in \brk[c]{0,\beta}.
\end{align*}
Thus, $f_i$ is convex and $\beta$-smooth. We lower bound the first derivative by
\begin{align*}
    \nabla f_i(x) = -G + \beta \int_{-\infty}^x \IND{\exists ~ j \leq i ~\text{s.t.}~ z \in [\ymin{n_j},\ymax{n_j}]} dz \geq -G,
\end{align*}
and by the definition of $M$, $\nabla f_i(x) < -G/2$.
Hence for all $x \in \R$ we have $\nabla f_i(x) \in [-G,-G/2)$, so $f_i$ is $G$-Lipschitz over $\R$.
\end{proof}

Next, we analyze the iterations of \agd on $f_i$ for any given $0 \leq i \leq M$. 
Fix such index $i$ and consider $(x_t,y_t)=\agd(f_i,0,t)$ and $(\xt_t,\yt_t)=\agd(f_i,\epsilon,t)$ for all $t \leq T$ for some $T \geq n_{i+1}$. 
We introduce the following compact notation for differences between the \agd terms related to the two sequences:
\begin{align*}
    \dltx{t} \eqdef x_t - \tilde{x}_t, \qquad
    \dlty{t} \eqdef y_t - \tilde{y}_t, \qquad
    \dltf{t} \eqdef \nabla f_i(x_t) - \nabla f_i(\tilde{x}_t), \qquad
    \dltm{t} \eqdef m_t - \tilde{m}_t.
\end{align*}
From the update rules of \agd (\cref{eq:xstep,eq:ystep,eq:mstep}), we have that
\begin{align}
    \dltx{t} &= \dlty{t-1} - \eta \dltf{t-1}, \label{eq:dxstep} \\
    \dlty{t} &= \dltx{t} + \dltm{t} = \dlty{t-1} - \eta \dltf{t-1} + \dltm{t}, \label{eq:dystep} \\
    \dltm{t} &= \gamma_{t}(\dltm{t-1}-\eta\dltf{t-1}). \label{eq:dmstep}
\end{align}
Our next lemma below describes the evolution of the differences $\dltf{t}$ and $\dltm{t}$ in terms of $\dlty{t}$.

\begin{lemma} \label{lemma:dltfm}
For all $t \leq T$,
\begin{align*}
    \dltf{t} =
    \begin{cases}
    \beta\dlty{t} & \mbox{if } t \in \brk[c]{n_j}_{j=1}^{i}; \\
    0 & \mbox{otherwise,}
    \end{cases}
\qquad
\text{and}
\qquad
    \dltm{t} =
    \begin{cases}
    \gamma_t(\dltm{t-1} - \eta\beta\dlty{t-1}) & \mbox{if } t \in \brk[c]{n_j+1}_{j=1}^{i}; \\
    \gamma_t\dltm{t-1} & \mbox{otherwise.}
    \end{cases}
\end{align*}
\end{lemma}

The following lemma summarise the evolution of the distance between the sequences $y_t$ and $\yt_t$ at steps $t \in \{n_j\}_{1 \leq j \leq i+1}$ and for $t>n_{i+1}$. The exponential growth is achieved by a balance between the difference in momentum terms and the difference between the sequences.

\begin{lemma} \label{lemma:evolution}
For the difference terms $\dlty{t}$, we have the following:
\begin{enumerate}[label=(\roman*),nosep]
    \item
    For all $1 \leq j \leq i$, it holds that
    $
        \frac{2}{3} \eta\beta\abs{\dlty{n_j}} \leq \abs{\dltm{n_j+1}} \leq \frac{1}{5} \eta\beta\abs{\dlty{n_{j+1}}}.
    $
    \item
    For all $0 \leq j \leq i$, it holds that $\abs{\dlty{n_{j+1}}}=\ymax{n_{j+1}}-\ymin{n_{j+1}} \geq 3^{j} \epsilon$.
    \item
    For all $t > n_{i+1}$, it holds that $\abs{\dlty{t}} \geq \abs{\dlty{n_{i+1}}}$.
\end{enumerate}
\end{lemma}

Finally, we can show that $\fmp$ is well-defined by proving that $M$ is finite in the following lemma. The bound of $M$ also indicate that after $O(\log\tfrac{1}{\epsilon})$ steps the two trajectories $y_t,\tilde{y}_t$ reach a constant distance.
\begin{lemma}\label{lemma:finite_m}
It holds that $1 \leq M \leq \ln \frac{3G}{2\beta\epsilon}$ (in particular, $M$ is finite), and $\ymax{n_{M+1}}-\ymin{n_{M+1}} \geq \frac{G}{3\beta}$.
\end{lemma}

Now we can return to our $\fmp$. First, we show it indeed posses the basic properties for \cref{thm:main_init_stab}.

\begin{lemma}\label{lemma:f_props}
The function $\fmp$ is convex, $\beta$-smooth, $G$-Lipschitz and attains a minimum $x^{\star} \in \argmin_x f(x)$ s.t. 
$
    \abs{x_0-x^\star} 
    = 
    O\brk!{(\ifrac{G}{\eta\beta^2})\log\brk{\ifrac{G}{\beta \epsilon}}^2}
$
for $x_0 \in \set{0,\epsilon}$.
\end{lemma}

The final lemma we require shows that the distance between the two trajectories is the same for $\fmp$ and~$f_M$. This holds true since the two functions coincide for $x \leq p$, after which the iterates reach a plateau which induces similar stability dynamics as the linear part of $f_M$ at $x > p$.

\begin{lemma}\label{lemma:diff_equality}
Let $(x_t,y_t)=\agd(\fmp,0,t,\eta)$ and $(\xt_t,\yt_t)=\agd(\fmp,\epsilon,t,\eta)$ be the iterations of \agd on $\fmp$ from our initialization points. Similarly, for $f_M$, let $(\xh_t,\yh_t)=\agd(f_M,0,t,\eta)$ and $(\xb_t,\yb_t)=\agd(f_M,\epsilon,t,\eta)$. 
Then for all $t$, we have that
$
    x_t-\xt_t = \xh_t-\xb_t
    \;\text{and}\;
    y_t-\yt_t = \yh_t-\yb_t
    .
$
\end{lemma}

We defer the proofs of \cref{lemma:dltfm,lemma:evolution,lemma:finite_m,lemma:f_props,lemma:diff_equality} to \cref{appendix:proofs}, and proceed to prove our main result.

\begin{proof}[of \cref{thm:main_init_stab}]
Based on \cref{lemma:diff_equality}, it suffices to show that the lower bound holds for the function $f_M$.
Let $c_1=\tfrac{1}{11}\ln(3),c_2=\tfrac{4}{5}3^{-3}$.
Let $t=n_i$ for some $i \geq 1$.
The first case we will deal with is when $i \leq M + 1$.
We already established with \cref{lemma:evolution} that $\abs{\dlty{n_{i}}} \geq 3^{i-1} \epsilon$. Since $t=n_i=(i+2)\ceil{10/\eta\beta}$, $i \geq \ifrac{\eta\beta t}{11} - 2$.
Hence,
\begin{align*}
    \abs{\dlty{n_{i}}}
    &\geq 3^{\eta\beta t/11 - 3} \epsilon
    \quad\implies\quad
    \abs{\dlty{n_{i}}}
    \geq 
    \frac{5}{4}c_2 e^{c_1 \eta\beta T}\epsilon.
\end{align*}
To relate to $\abs{x_t-\xt_t}$,
\begin{align*}
    \abs{x_{t}-\xt_{t}}
    &= \abs{\dltx{n_{i}}} \\
    &\geq \abs{\dlty{n_{i}}} - \abs{\dltm{n_{i}}} \tag{\cref{eq:dystep}}  \\
    &\geq \abs{\dlty{n_{i}}} - \abs{\dltm{n_{i-1}+1}} \prod_{t=n_{i-1}+2}^{n_{i}}\gamma_t \tag{\cref{lemma:dltfm} for $t=n_{i},\dots,n_{i-1}+2$} \\
    &\geq \abs{\dlty{n_{i}}} - \abs{\dltm{n_{i-1}+1}} \tag{Since $t \geq 1 \Rightarrow 0 \leq \gamma_t \leq 1$} \\
    &\overset{(*)}{\geq} \abs{\dlty{n_{i}}}\brk2{1-\frac{\eta\beta}{5}} \\
    &\geq \frac{4}{5} \abs{\dlty{n_{i}}}
    \geq c_2 e^{c_1 \eta\beta T}\epsilon. \tag{$\eta \leq \frac{1}{\beta}$}
\end{align*}
Here, $(*)$ follows from \cref{lemma:evolution} if $i>1$ and the case of $i=1$ follows by combining \cref{lemma:dltfm} and $\dltm{1}=0$ which implies that $\dltm{n_0+1}=0$.
If $t>n_{M+1}$ (includes the case of $i > M+1$ and $t>\ceil{10/\eta\beta}\brk!{\ln\frac{3G}{2\beta\epsilon}+3}$ from \cref{lemma:finite_m}), since $t-1 \geq n_{M+1}$,
\begin{align*}
    \abs{x_{t}-\xt_{t}}
    &= \abs{\dlty{t-1} - \eta \dltf{t-1}} \tag{\cref{eq:dxstep}} \\
    &= \abs{\dlty{t-1}} \tag{\cref{lemma:dltfm}} \\
    &\geq \abs{\dlty{n_{M+1}}}. \tag{\cref{lemma:evolution}}
\end{align*}
And using \cref{lemma:finite_m} we conclude that
$
    \abs{x_{t}-\xt_{t}} \geq \frac{G}{3\beta}
    .
$
Hence, with \cref{lemma:f_i}, $f_M$ holds all properties of \cref{thm:main_init_stab} beside attaining a minimum, and using \cref{lemma:diff_equality,lemma:f_props}, $\fmp$ posses all the properties needed for \cref{thm:main_init_stab}.
\end{proof}

\section{Uniform Stability of \agd} \label{sec:uniform}

In this section we present our second main result, regarding the uniform stability of (full-batch) \agd. This is given formally in the following theorem.

\begin{theorem} \label{thm:main_as}
For any $G,\beta,\eta$ and $n \geq 4$ such that $0 < \eta \leq \ifrac{1}{\beta}$, there exists a loss function $\ell(w;z)$ that is convex, $\beta$-smooth and $G$-Lipschitz in $w$ (for every $z\in \Z$) and universal constants $c_3,c_4>0$, such that the uniform stability of $T$-steps full-batch \agd with step size $\eta$ is
\begin{align*}
    \unif{\agd_T}{\ell}(n)
    \geq \min\brk[c]1{\tfrac{G^2}{3\beta}, c_4 e^{c_3 \eta\beta T} \tfrac{\beta \eta^2 G^2}{n}} ,
    \qquad\qquad
    \forall ~ T \in \brk[c]!{\ceil{\tfrac{10}{\eta\beta}\tfrac{n}{n-3}}(i+2) ~:~ i = 1,2,\ldots}
    ,
\end{align*}
Furthermore, for all $T \geq \ceil{\frac{40}{\eta\beta}} \brk!{ \ln\frac{6G}{\beta\epsilon}+3}$ it holds that $\unif{\agd_T}{\ell}(n) \geq \frac{G^2}{3\beta}$.
\end{theorem}

The comments following \cref{thm:main_init_stab} regarding the exponential blowup and the fluctuating behaviour also apply here. 
Note also the perhaps surprising inverse dependence on $\beta$ ($\beta\eta^2$ is also $O(\ifrac{1}{\beta})$). The dependence can be explained by the fact that smooth optimization over a highly non-smooth yet still $G$-Lipschitz function must have a small step size (with $\eta\leq \ifrac{1}{\beta}$) which improves stability.

\paragraph{Construction.}

We denote the given parameters for the theorem with $\hat{G},\hat{\beta},\hat{\eta},n$.
We will use the construction from \cref{sec:initialization} with the properties of \cref{thm:main_init_stab} in order to create a loss function and samples which will have the same optimization for $t \geq 1$.
For the construction we define the following setting of $G,\beta,\eta,\epsilon$:
\begin{align*}
    G = \hat{G},
    \qquad
    \beta = \hat{\beta},
    \qquad
    \eta = \frac{n-3}{n}\hat{\eta},
    \qquad
    \epsilon = \frac{\beta \eta^2 G}{n-3}.
\end{align*}
Using these parameters, we obtain $\fmp$ from the construction of \cref{sec:initialization}. As we proved in the previous section, this is the function for which \cref{thm:main_init_stab} holds.
Note that for $T<n_1$, the lower bounds already holds even for quadratics, as we show in \cref{sec:lower_agd_quad_unif}.
We also define the following functions,
\begin{align} \label{eq:g2}
    \ell(w;1) &\eqdef 0, \\
    \ell(w;2) &\eqdef -\beta\eta G w + \beta \int_{-\infty}^w \int_{-\infty}^{y} \IND{z \in \brk[s]*{0,\eta G}} dz dy, 
\end{align}
and further let $\ell(w;3) = -Gw$, $\ell(w;4) = Gw$, and $\ell(w;5) = f(w)$.
Note that all are convex, $\beta$-smooth and $G$-Lipschitz.
Let $\Z=\brk[c]{1,2,3,4,5}$ be our sample space, we consider the loss function $\ell(w;z)$ over $\R \times \Z$.
Our samples of interest are $S=(1,3,4,5,\dots,5)\in\Z^n$ and $S'=(2,3,4,5,\dots,5)\in\Z^n$. 
Thus, the empirical risks $R_S, R_{S'}$ corresponding to $S, S'$ are given by
\begin{align}
    R_{S}(w) &= \frac{1}{n}g_1(w) + \frac{n-3}{n}\fmp(w) = \frac{n-3}{n}\fmp(w), \label{eq:rs} \\
    R_{S'}(w) &= \frac{1}{n}g_2(w) + \frac{n-3}{n}\fmp(w). \label{eq:rst}
\end{align}

\paragraph{Analysis.}

The key lemma below shows that we constructed a scenario that reduces the problem of analyzing the uniform stability of full-batch \agd to analyzing its initialization stability on the function $\fmp$ we constructed in \cref{sec:initialization}.

\begin{lemma} \label{lemma:reduction}
For all $t=1,2,\ldots$, we have that
\begin{align*}
    \agd(\fmp,0,t,\eta) &= \agd(R_S,0,t,\hat{\eta}), \\
    \agd(\fmp,\epsilon,t,\eta) &= \agd(R_{S'},0,t,\hat{\eta}).
\end{align*}
\end{lemma}
Using this key lemma, \cref{thm:main_as} follows immediately by setting $c_3=c_1/4$ and $c_4=c_2/4$, for the samples $S,S'$ we defined and $z=3$. As $\ell(w;3)=-Gw$,
\begin{align*}
    \abs({\ell(x_T;3)-\ell(\xt_T;3)}
    &= G\abs{x_T - \xt_T},
\end{align*}
and by using the definitions of $G,\beta,\eta,\epsilon$ we get the two cases of \cref{thm:main_init_stab} which lower bound $\abs{x_T - \xt_T}$ as needed.

Below we give a proof sketch for the second equality of \cref{lemma:reduction}, deferring the full proofs of the lemma and \cref{thm:main_as} to \cref{sec:uniform-proofs}.

\begin{proof}[of \cref{lemma:reduction} (sketch)]
We sketch the proof of the second equality stated in the lemma; the proof of the first equality is simpler and follows similar lines (details are given in \cref{sec:uniform-proofs}). 

The proof proceeds via induction over the iterations.
Let $(x_t,y_t)=\agd(\fmp,\epsilon,t,\eta)$ and $(\xt_t,\yt_t)=\agd(R_{S'},0,t,\hat{\eta})$ for $t \leq T$.
For $t=1$,
\begin{align*}
    \xt_1
    &= \yt_0 - \hat{\eta}\nabla R_{S'}(y_0) \tag{\cref{eq:xstep}} \\
    &= 0 - \hat{\eta} \frac{1}{n} \nabla g_2(0) - \hat{\eta} \frac{n-3}{n} \nabla \fmp(0) \tag{$\yt_0=0$ and \cref{eq:rst}} \\
    &= -\eta \frac{1}{n-3} \nabla g_2(0) - \eta \nabla \fmp(0) \tag{def of $\eta$, $\eta=\hat{\eta}\frac{n-3}{n}$} \\
    &= \epsilon - \eta \nabla \fmp(0) \tag{$\nabla g_2(0)=-\beta\eta G$ and $\epsilon=\frac{\beta\eta^2 G}{n-3}$} \\
    &= x_1 + \eta (\nabla \fmp(\epsilon) - \nabla \fmp(0)). \tag{\cref{eq:xstep} and $y_0=\epsilon$}
\end{align*}
In \cref{clm:nabla0eps} we show that $\nabla \fmp(\epsilon) = \nabla \fmp(0) = -G$, thus $\xt_1=x_1$.
Since $\gamma_1=0$, it follows from \cref{eq:ystep} that $\yt_1 = \xt_1$ and similarly $x_1=y_1$, hence $\yt_1=y_1$.
For $t>1$,
\begin{align*}
    \xt_{t}
    &= \yt_{t-1} - \hat{\eta} \nabla R_{S'}(\yt_{t-1}) \tag{\cref{eq:xstep}} \\
    &= \yt_{t-1} - \hat{\eta} \frac{n-3}{n}\nabla \fmp(\yt_{t-1}) - \hat{\eta} \frac{1}{n}\nabla g_2(\yt_{t-1}) \tag{\cref{eq:rst}} \\
    &= \yt_{t-1} - \eta \nabla \fmp(\yt_{t-1}) - \hat{\eta} \frac{1}{n}\nabla g_2(\yt_{t-1}) \tag{def of $\eta$, $\eta=\hat{\eta}\frac{n-3}{n}$} \\
    &= y_{t-1}-\eta\nabla \fmp(y_{t-1}) - \hat{\eta} \frac{1}{n}\nabla g_2(y_{t-1}) \tag{induction assumption} \\
    &= x_t - \hat{\eta} \frac{1}{n}\nabla g_2(y_{t-1}) \tag{\cref{eq:xstep}}.
\end{align*}
We need to show that $\nabla g_2(y_{t-1})=0$.
Since by our construction, $\nabla \fmp(x) \leq 0$ for all $x \in \R$, we observe only negative gradients and the iterations always move in the positive direction. Hence, $y_{t-1} \geq y_1$.
Since
$
    y_{t-1}
    \geq y_1
    = x_1 + \underset{=0}{\gamma_1}(x_1-x_0)
    = y_0 - \eta \nabla \fmp(y_0)
    = \epsilon - \eta \nabla \fmp(\epsilon)
    = \epsilon + \eta G,
$
it follows that
\begin{align*}
    \nabla g_2(y_{t-1})
    &= -\beta\eta G + \beta \int_{-\infty}^{y_{t-1}} \IND{z \in \brk[s]*{0,\eta G}} dz
    = -\beta\eta G + \beta \eta G
    =0,
\end{align*}
hence $\xt_t=x_t$.
Since by the induction assumption, $\xt_{t-1}=x_{t-1}$, it follows from \cref{eq:ystep} that $\yt_t=y_t$, and we finished our induction.
\end{proof}

%% file: appendix.tex
\section{Proofs for \cref{sec:initialization}}
\label{appendix:proofs}

For the proofs in this section, we require the following lemma that states that
the iterates of \agd over consecutive functions $f_{i-1}$ and $f_i$ in the
construction in \cref{sec:initialization} are identical up to iteration $t=n_i$.
Hence, for $j \leq i$, the iterates of \agd over $f_i$ and $f_j$ are identical
up to $t=n_{j+1}$.

\begin{lemma} \label{lemma:consistency}
For all $1 \leq i \leq M$ and $x_0 \in \brk[c]{0,\epsilon}$ we have
\begin{align*}
    \agd(f_i,x_0,t) = \agd(f_{i-1},x_0,t),
    \qquad
    \forall ~ t \leq n_i
    .
\end{align*}
\end{lemma}

\subsection{Proof of \cref{lemma:consistency}}

We will need the following technical claims (proofs below).

\begin{claim} \label{clm:upper_bound_diff}
For all $1 \leq i \leq M$ we have $\ymax{n_i}-\ymin{n_i}<\frac{G}{2\beta}$.
\end{claim}

\begin{claim} \label{clm:min_step}
Let $f: \R \rightarrow \R$ a convex, $\beta$-smooth function such that for all $x \in R$, $\nabla f(x) < -\frac{G}{2}$. Let $(x_t,y_t)=\agd(f,x_0,t,\eta)$ for all $t \leq T$, for some $T$ and step size $\eta$. Then $y_t>y_{t-1}$ and if $t > \ceil{20/\eta\beta}$, then $y_t >y_{t-1}+ \frac{2G}{\beta}$.
\end{claim}

We now proceed to prove \cref{lemma:consistency}.

\begin{proof}
Let $(x_t,y_t)=\agd(f_i,x_0,t)$ and $(\xt_t,\yt_t)=\agd(f_{i-1},x_0,t)$ for all $t \leq n_i$. We first note that for all $x \in (-\infty,\ymin{n_i})$,
\begin{align*}
    \nabla f_{i}(x) &= -G + \beta\int_{-\infty}^{x} \IND{\exists j \in [i] \text{ s.t. } z \in [\ymin{n_j},\ymax{n_j}]} dz \\
    &= -G + \beta\int_{-\infty}^{x} \IND{\exists j \in [i-1] \text{ s.t. } z \in [\ymin{n_j},\ymax{n_j}]} dz
    = \nabla f_{i-1}(x).
\end{align*}
We will use induction over $t$. For $t=0$, $y_0=x_0=\xt_0=\yt_0$ is the starting point.
For $t=1$,
\begin{align*}
    x_1
    &= y_0-\eta\nabla f_i(y_0) \tag{\cref{eq:xstep}} \\
    &= \yt_0-\eta\nabla f_i(\yt_0) \tag{since $y_0=\yt_0$} \\
    &= \xt_1-\eta(\nabla f_i(\yt_0)-\nabla f_{i-1}(\yt_0)). \tag{\cref{eq:xstep}}
\end{align*}
From \cref{clm:min_step}, $\ymin{n_i} > \frac{G}{2\beta}$, and given the assumption that $\epsilon < \frac{G}{2\beta}$, $\yt_0 \in (-\infty,\ymin{n_i})$, thus,
\begin{align*}
    \nabla f_i(\yt_0)=\nabla f_{i-1}(\yt_0).
\end{align*}
This means that $x_1=\xt_1$, and since $\gamma_1=0$,
\begin{align*}
    y_1=x_1+\gamma_1(x_1-x_0)=x_1=\xt_1=\yt_1.
\end{align*}
For $t > 1$,
\begin{align*}
    x_t
    &= y_{t-1}-\eta\nabla f_i(y_{t-1}) \tag{\cref{eq:xstep}} \\
    &= \yt_{t-1}-\eta\nabla f_i(\yt_{t-1}) \tag{induction assumption} \\
    &= \xt_t-\eta(\nabla f_i(\yt_{t-1})-\nabla f_{i-1}(\yt_{t-1})). \tag{\cref{eq:xstep}}
\end{align*}
Again, we want to show that $\yt_{t-1} \in (-\infty,\ymin{n_i})$.
Combining \cref{clm:min_step} with the fact that $\yt_{n_i} \in \brk[c]{\ymin{n_i},\ymax{n_i}}$,
\begin{align*}
    \yt_{t-1}
    &\leq \yt_{n_i-1} \\
    &< \yt_{n_i} - \frac{G}{2\beta} \\
    &\leq \ymax{n_i} - \frac{G}{2\beta},
\end{align*}
and from \cref{clm:upper_bound_diff},
\begin{align*}
    \ymax{n_i} - \frac{G}{2\beta}
    &< \ymin{n_i}.
\end{align*}
Thus, $\yt_t < \ymin{n_i}$ and as before, $\nabla f_i(\yt_{t-1})=\nabla f_{i-1}(\yt_{t-1})$, hence, $x_t=\xt_t$. Lastly, $y_t=\yt_t$ follows directly from the induction assumption and $x_t=\xt_t$ by \cref{eq:ystep}.
\end{proof}

Below we prove the claims used in the proof above.

\begin{proof}[of \cref{clm:upper_bound_diff}]
From the definition of $M$, $\max_x \nabla f_i(x) <-G/2$. 
Specifically,
\begin{align*}
    -G/2 
    >
    \nabla f_{i}(\ymax{n_i})
    &= -G + \beta\int_{-\infty}^{\ymax{n_i}} \IND{\exists j \in [i] \text{ s.t. } z \in [\ymin{n_j},\ymax{n_j}]} dz
    .
\end{align*}
Thus,
\begin{align*}
    \frac{G}{2\beta}
    &> \int_{-\infty}^{\ymax{n_i}} \IND{\exists j \in [i] \text{ s.t. } z \in [\ymin{n_j},\ymax{n_j}]} dz \\
    &\geq \int_{-\infty}^{\ymax{n_i}} \I\brk[s]{z \in [\ymin{n_i},\ymax{n_i}]} dz \\
    &= \ymax{n_i}-\ymin{n_i}.
    \qedhere
\end{align*}
\end{proof}

In order to prove \cref{clm:min_step} we need another technical result (proof in \cref{appendix:technical}).

\begin{claim} \label{clm:cool_sum2}
Let $\gamma_t=\frac{t-1}{t+2}$ for $t \geq 1$. Then $\sum_{k=2}^{t}\prod_{j=k}^{t} \gamma_j \geq \frac{(t-1)^2}{4(t+2)}$.
\end{claim}

\begin{proof}[of \cref{clm:min_step}]
Since for all $x \in \R$, $\nabla f(x) < -G/2$, from \cref{eq:mstep},
\begin{align*}
    m_t
    &= \gamma_t(m_{t-1}-\eta\nabla f(y_{t-1}))
    \geq \gamma_t\brk*{m_{t-1}+\frac{\eta G}{2}}.
\end{align*}
Unfolding the recursion until reaching $m_1=\gamma_1(x_1-x_0)=0\cdot(x_1-x_0)=0$, we obtain
\begin{align*}
    m_t
    &\geq \frac{\eta G}{2} \sum_{k=2}^t \prod_{j=k}^{t} \gamma_j \geq 0.
\end{align*}
Hence,
\begin{align*}
    y_{t}
    &= y_{t-1}-\eta\nabla f(y_{t-1}) + m_{t} \tag{\cref{eq:ystep}} \\
    &\geq y_{t-1} + \frac{\eta G}{2} > y_{t-1}.
\end{align*}
If $t>\ceil{10/\eta\beta}$, using \cref{clm:cool_sum2},
\begin{align*}
    m_t &\geq \frac{\eta G}{2} \sum_{k=2}^t \prod_{j=k}^{t} \gamma_j \\
    &\geq \frac{\eta G(t-1)^2}{8(t+2)} \\
    &\geq \frac{\eta G(20/\eta\beta)^2}{8(20/\eta\beta+3)} \\
    &= \frac{G}{\beta} \frac{2.5}{1+0.15\eta\beta}
    > \frac{2G}{\beta},
\end{align*}
and we finish with
\begin{align*}
    y_{t}
    &= y_{t-1}-\eta\nabla f(y_{t-1}) + m_{t} \tag{\cref{eq:ystep}} \\
    &> y_{t-1} + \frac{2G}{\beta}.
    \qedhere
\end{align*}
\end{proof}

\subsection{Proof of \cref{lemma:dltfm}}

In order to prove \cref{lemma:dltfm} we need the following claim (proved below).

\begin{claim} \label{clm:eta_t_ni}
For all $i \in [M]$, for all $x_0 \in \brk[c]{0,\epsilon}$, let $(x_t,y_t)=\agd(f_i,x_0,t)$ for all $t \leq T$ for some $T$. Then for all $j \in [i]$, $t<n_j \Rightarrow y_t<\ymin{n_j}$ and $t>n_j \Rightarrow y_t>\ymax{n_j}$.
\end{claim}

\begin{proof}[of \cref{lemma:dltfm}]
The proof of the first property considers two different cases:
\begin{itemize}
    \item
    If $t=n_k$ for some $k \in [i]$, we know from \cref{lemma:consistency} that
    \begin{align*}
        \ymin{n_k}&=\min(y_{n_k},\yt_{n_k}), \\
        \ymax{n_k}&=\max(y_{n_k},\yt_{n_k}).
    \end{align*}
   Using this property,
    \begin{align*}
        \dltf{n_k}
        = \beta \int_{\yt_{n_k}}^{y_{n_k}} \IND{\exists j \in [i] \text{ s.t. } z \in [\ymin{n_j},\ymax{n_j}]} dz
        = \beta\dlty{n_k}.
    \end{align*}
    
    \item
    If $t \not\in \brk[c]{n_j}_{j=1}^{i}$, then from \cref{clm:eta_t_ni} we know that they are both inside an interval of the form $(\ymax{n_k},\ymin{n_{k+1}})$ (or $(-\infty,\ymin{n_{1}})$ or $(\ymax{n_i},\infty)$), in which case they are inside an interval with second derivative of $0$, and using the mean value theorem, there exists some $y^{mid} \in (y_t,\yt_t)$, for which
    \begin{align*}
        \dltf{t}=\dlty{t}\cdot\nabla^2 f_i(y^{mid})=\dlty{t}\cdot 0=0.
    \end{align*}
\end{itemize}
The second property comes from the first together with \cref{eq:dmstep},
\begin{align*}
    \dltm{t}
    &= \gamma_t(\dltm{t-1}-\eta\dltf{t-1}),
\end{align*}
which completes the proof.
\end{proof}

\begin{proof}[of \cref{clm:eta_t_ni}]
For some $j \in [i]$ and $t \in [T]\cup\brk[c]{0}$:
\begin{itemize}
    \item
    If $t<n_j$, then from \cref{clm:min_step},
    \begin{align*}
        y_{t} &\leq y_{n_j-1} < y_{n_j} - \frac{G}{2\beta}.
    \end{align*}
    From \cref{lemma:consistency}, $y_{n_j} \in \brk[c]{\ymin{n_j},\ymax{n_j}}$, and combining with \cref{clm:upper_bound_diff},
    \begin{align*}
        y_{n_j} - \frac{G}{2\beta} \leq \ymax{n_j}-\frac{G}{2\beta} < \ymin{n_j}.
    \end{align*}
    
    \item
    If $t > n_j$, then similarly,
    \begin{align*}
        y_t \geq y_{n_j+1} > y_{n_j} + \frac{G}{2\beta} \geq \ymin{n_j}+\frac{G}{2\beta} > \ymax{n_i}.
    \end{align*}
\end{itemize}
This concludes the proof.
\end{proof}

\subsection{Proof of \cref{lemma:evolution}}

The three parts of the lemma are defined separately in the following lemmas (proved separately afterwards).
The first analyzes the evolution of the distance between the iterates $y_t$ and $\tilde{y}_t$ at steps $t \in \{n_j\}_{1 \leq j \leq i+1}$ and relates it to the difference in momentum terms.

\begin{lemma} \label{lemma:ymy}
For all $1 \leq j \leq i$,
\begin{align*}
    \frac{2}{3} \eta\beta\abs{\dlty{n_j}} \leq \abs{\dltm{n_j+1}} \leq \frac{1}{5} \eta\beta\abs{\dlty{n_{j+1}}}.
\end{align*}
\end{lemma}

The second lemma, derived from the first, shows that the difference $\dlty{t}$ exhibits an exponential blowup between steps $t=n_j$ and $t=n_{j+1}$.

\begin{lemma} \label{lemma:exp_growth}
For all $0 \leq j \leq i$, it holds that $\abs{\dlty{n_{j+1}}}=\ymax{n_{j+1}}-\ymin{n_{j+1}} \geq 3^{j} \epsilon$.
\end{lemma}

In order to bound the iterations after $t=n_{i+1}$, the following lemma shows that the distance between the iterates $y_t$ and $\tilde{y}_t$ after the exponential growth phase (when $t > n_{i+1}$) does not decrease.

\begin{lemma} \label{lemma:afterwards}
For all $t > n_{i+1}$, it holds that $\abs{\dlty{t}} \geq \abs{\dlty{n_{i+1}}}$.
\end{lemma}

The three lemmas are proved in the following sections.

\subsection{Proof of \cref{lemma:ymy}}

We will need the following claim (proof in \cref{appendix:technical}).

\begin{claim} \label{clm:cool_sum}
Let $\gamma_t=\frac{t-1}{t+2}$. Then for $n,m \geq 1$,
\begin{align*}
    \sum_{k=n}^{m}\prod_{t=n}^{k-1} \gamma_{t+1} \geq \frac{n}{2} \brk*{1-\frac{n^2}{(m+1)^2}}.
\end{align*}
\end{claim}

\begin{proof}[of \cref{lemma:ymy}]
First let us assume that $\frac{2}{3}\eta\beta\abs{\dlty{n_j}} \leq \abs{\dltm{n_j+1}}$ for some $j \in [i]$. Thus,
\begin{align*}
    \dlty{n_{j+1}}
    &= \dlty{n_{j+1}-1} - \eta\dltf{n_{j+1}-1} + \dltm{n_{j+1}} \tag{\cref{eq:dystep}} \\
    &= \dlty{n_j} - \eta \sum_{k=n_{j}}^{n_{j+1}-1} \dltf{k} + \sum_{k=n_{j}+1}^{n_{j+1}} \dltm{k} \tag{\cref{eq:dystep} multiple times} \\
    &= \dlty{n_j} - \eta\beta\dlty{n_j} + \sum_{k=n_{j}+1}^{n_{j+1}} \dltm{k} \tag{\cref{lemma:dltfm}} \\
    &= \dlty{n_j}(1-\eta\beta) + \sum_{k=n_{j}+1}^{n_{j+1}} \dltm{k}=(*).
\end{align*}
Using \cref{lemma:dltfm} ($t \in [T]/\brk[c]{n_j+1}_{j=1}^{i} \Rightarrow \dltm{t}=\gamma_t\dltm{t-1}$) recursively on $\dltm{k}$,
\begin{align*}
    (*) = \dlty{n_j}(1-\eta\beta) + \dltm{n_{j}+1}\sum_{k=n_{j}+1}^{n_{j+1}}\prod_{t=n_j+1}^{k-1} \gamma_{t+1}.
\end{align*}
Thus, since $\eta\beta \leq 1$ and for $t \geq 0$, $\gamma_{t+1}\geq 0$,
\begin{align*}
    \abs{\dlty{n_{j+1}}}
    &\geq \abs{\dltm{n_{j}+1}} \sum_{k=n_{j}+1}^{n_{j+1}}\prod_{t=n_j+1}^{k-1} \gamma_{t+1} - (1-\eta\beta)\abs{\dlty{n_j}}.
\end{align*}
Using our assumption that $\frac{2}{3}\eta\beta\abs{\dlty{n_j}} \leq \abs{\dltm{n_j+1}}$,
\begin{align*}
    \abs{\dlty{n_{j+1}}}
    &\geq \brk*{\sum_{k=n_{j}+1}^{n_{j+1}}\prod_{t=n_j+1}^{k-1} \gamma_{t+1}-\frac{3(1-\eta\beta)}{2\eta\beta}}\abs{\dltm{n_j+1}} \\
    &\geq \brk*{\sum_{k=n_{1}+1}^{n_{2}}\prod_{t=n_1+1}^{k-1} \gamma_{t+1}-\frac{3(1-\eta\beta)}{2\eta\beta}}\abs{\dltm{n_j+1}}.
\end{align*}
Using \cref{clm:cool_sum},
\begin{align*}
    \sum_{k=n_{1}+1}^{n_{2}}\prod_{t=n_1+1}^{k-1} \gamma_{t+1}
    &\geq \frac{n_1+1}{2}\brk*{1-\frac{(n_1+1)^2}{(n_2+1)^2}} \\
    &\geq \frac{30}{2\eta\beta}\brk*{1-\frac{3^2}{4^2}} \tag{$n_i=\ceil{10/\eta\beta}(i+2)$} \\
    &\geq \frac{13}{2\eta\beta}.
\end{align*}
Plugging it gives us
\begin{align*}
    \abs{\dlty{n_{j+1}}}
    &\geq \brk*{13-3(1-\eta\beta)}\frac{\abs{\dltm{n_j+1}}}{2\eta\beta} \\
    &\geq \frac{5}{\eta\beta}\abs{\dltm{n_j+1}}. \tag{$\eta\beta \geq 0$}
\end{align*}
So if $\frac{2}{3}\eta\beta\abs{\dlty{n_j}} \leq \abs{\dltm{n_j+1}}$ is true for some $j$, then $\abs{\dltm{n_j+1}} \leq \frac{1}{5} \eta\beta\abs{\dlty{n_{j+1}}}$ is also true.
We will show $\frac{2}{3}\eta\beta\abs{\dlty{n_j}} \leq \abs{\dltm{n_j+1}}$ by induction, which will conclude our proof.
\begin{align*}
    \abs{\dltm{n_1+1}}
    &= \gamma_{n_{1}+1}\abs{\dltm{n_{1}}-\eta\dltf{n_{1}}} \tag{\cref{eq:dmstep}} \\
    &= \gamma_{n_{1}+1}\abs{\dltm{n_{1}}-\eta\beta\dlty{n_{1}}} \tag{\cref{lemma:dltfm}} \\
    &= \gamma_{n_{1}+1}\abs2{\dltm{1}\prod_{k=2}^{n_1}\gamma_k-\eta\beta\dlty{n_{1}}} \tag{\cref{lemma:dltfm}} \\
    &= \gamma_{n_{1}+1}\abs{\eta\beta\dlty{n_{1}}} \tag{$\dltm{1}=0$} \\
    &\geq \frac{2}{3}\eta\beta\abs{\dlty{n_{1}}}. \tag{$t \geq 7 \Rightarrow \gamma_{t}\geq \frac{2}{3}$}
\end{align*}
Now for a given $j>1$,
\begin{align*}
    \dltm{n_{j}+1}
    &= \gamma_{n_{j}+1}(\dltm{n_{j}}-\eta\dltf{n_{j}}) \tag{\cref{eq:dmstep}} \\
    &= \gamma_{n_{j}+1}\dltm{n_{j}}-\gamma_{n_{j}+1}\eta\beta\dlty{n_{j}}, \tag{\cref{lemma:dltfm}}
\end{align*}
and again using \cref{lemma:dltfm} ($t \in [T]/\brk[c]{n_j+1}_{j=1}^{i} \Rightarrow \dltm{t}=\gamma_t\dltm{t-1}$),
\begin{align*}
    \dltm{n_{j}+1}
    &= \dltm{n_{j-1}+1} \prod_{k=n_{j-1}+2}^{n_j+1} \gamma_{k}-\gamma_{n_{j}+1}\eta\beta\dlty{n_j}.
\end{align*}
Thus, since $\gamma_t \geq 0$ for $t \geq 1$,
\begin{align*}
    \abs{\dltm{n_{j}+1}}
    &\geq \gamma_{n_{j}+1}\eta\beta\abs{\dlty{n_i}} - \abs{\dltm{n_{j-1}+1}}\prod_{k=n_{j-1}+2}^{n_j+1}\gamma_{k}.
\end{align*}
By the induction assumption, $\frac{2}{3}\eta\beta\abs{\dlty{n_{j-1}}} \leq \abs{\dltm{n_{j-1}+1}}$. Hence, $\abs{\dltm{n_{j-1}+1}} \leq \frac{1}{5} \eta\beta\abs{\dlty{n_{j}}}$ as we showed in the beginning of the proof. Thus,
\begin{align*}
    \abs{\dltm{n_{j}+1}}
    &\geq \gamma_{n_{j}+1}\eta\beta\abs{\dlty{n_j}} - \frac{1}{5}\eta\beta\abs{\dlty{n_{j}}}\prod_{k=n_{j-1}+2}^{n_j+1}\gamma_{k}
    \\
    &= \brk*{\gamma_{n_{j}+1}-\frac{1}{5}\prod_{k=n_{j-1}+2}^{n_j+1}\gamma_{k}}\eta\beta\abs{\dlty{n_j}} \\
    &\geq \brk*{\frac{9}{10}-\frac{1}{5}}\eta\beta\abs{\dlty{n_j}} \tag{$t \geq 28 \Rightarrow \gamma_{t}\geq \frac{9}{10}$} \\
    &\geq \frac{2}{3}\eta\beta\abs{\dlty{n_j}},
\end{align*}
and we concluded our inductive argument.
\end{proof}

\subsection{Proof of \cref{lemma:exp_growth}}

\begin{proof} \label{proof:minf}
Let $i \leq M$.
Let $(x_t,y_t)=\agd(f_i,0,t)$ and $(\xt_t,\yt_t)=\agd(f_i,\epsilon,t)$ for all $t \leq n_{i+1}$.
From \cref{lemma:ymy} we know that for all $j \in [i]$,
\begin{align*}
    \abs{\dlty{n_{j+1}}} \geq \frac{10}{3} \abs{\dlty{n_j}}
    \implies \abs{\dlty{n_{j+1}}} \geq 3^{j} \abs{\dlty{n_1}}.
\end{align*}
Using \cref{lemma:dltfm},
\begin{align*}
    \dlty{n_1}
    &= \dlty{n_1-1} - \eta\dltf{n_1-1} + \dltm{n_1} \tag{\cref{eq:dystep}} \\
    &= \dlty{n_1-1} + \dltm{1} \prod_{k=2}^{n_1} \gamma_k \tag{\cref{lemma:dltfm}} \\
    &= \dlty{n_1-1}. \tag{$m_1=\tilde{m}_1=0 \Rightarrow \dltm{1}=0$}
\end{align*}
Repeating this argument leads to $\dlty{n_1}=\dlty{0}=-\epsilon$.
Thus, $\abs{\dlty{n_{j+1}}} \geq 3^{j} \epsilon$. Since $\ymax{n_{i+1}}-\ymin{n_{i+1}}=\abs{\dlty{n_{i+1}}}$ as $\ymax{n_{i+1}}$ and $\ymin{n_{i+1}}$ are defined as the max and min of $\brk[c]{y_{n_{i+1}},\yt_{n_{i+1}}}$,
\begin{align*}
    \abs{\ymax{n_{i+1}}-\ymin{n_{i+1}}} \geq 3^{i} \epsilon,
\end{align*}
and since this is true for all $i \in [M]$, we showed our exponential growth.
All is left to show is that for all $j \leq i$, $\abs{\dlty{n_{j+1}}}=\ymax{n_{j+1}}-\ymin{n_{j+1}}$.
If $j=i$, then this is immediate since $\ymax{n_{j+1}},\ymin{n_{j+1}}$ are defined as the max and min of $y_{n_{j+1}},\yt_{n_{j+1}}$. If $j<i$, then based on \cref{lemma:consistency}, the iterations over $f_j$ and over $f_{j+1}$ are the same up to $n_{j+1}$. Invoking \cref{lemma:consistency} up to $i$ means that the iterations over $f_j$ and $f_i$ are the same up to $n_{j+1}$. Hence, for $j \leq i$, $\abs{\dlty{n_{j+1}}}=\ymax{n_{j+1}}-\ymin{n_{j+1}}$.
\end{proof}

\subsection{Proof of \cref{lemma:afterwards}}

\begin{proof}
We start by showing that
\begin{align*}
    \sigma(\dlty{n_{i+1}})=\sigma(\dltm{n_i+1}),
\end{align*}
where $\sigma$ is the sign function.
Using \cref{eq:dystep} recursively from $n_{i+1}$ to $n_{i}$,
\begin{align*}
    \dlty{n_{i+1}}
    &= \dlty{n_{i+1}-1} - \eta\dltf{n_{i+1}-1} + \dltm{n_{i+1}} \tag{\cref{eq:dystep}} \\
    &= \dlty{n_i} - \eta \sum_{k=n_{i}}^{n_{i+1}-1} \dltf{k} + \sum_{k=n_{i}+1}^{n_{i+1}} \dltm{k} \tag{\cref{eq:dystep} multiple times} \\
    &= \dlty{n_i} - \eta\beta\dlty{n_i} + \sum_{k=n_{i}+1}^{n_{i+1}} \dltm{k} \tag{\cref{lemma:dltfm}} \\
    &= \dlty{n_i}(1-\eta\beta) + \sum_{k=n_{i}+1}^{n_{i+1}} \dltm{k}.
\end{align*}
Using \cref{lemma:dltfm} ($t \in [T]/\brk[c]{n_j+1}_{j=1}^{i} \Rightarrow \dltm{t}=\gamma_t\dltm{t-1}$) recursively on $\dltm{k}$,
\begin{align*}
    \dlty{n_{i+1}} = \dlty{n_i}(1-\eta\beta) + \dltm{n_{i}+1}\sum_{k=n_{i}+1}^{n_{i+1}}\prod_{t=n_i+1}^{k-1} \gamma_{t+1}.
\end{align*}
Now we use the fact that $\abs{\dlty{n_{i+1}}} \geq 3\abs{\dlty{n_{i}}}$ from \cref{lemma:ymy},
\begin{align*}
    3\abs{\dlty{n_{i}}}
    &\leq \abs{\dlty{n_{i+1}}}
    = \sigma(\dlty{n_{i+1}})\dlty{n_{i+1}} \\
    &= \sigma(\dlty{n_{i+1}}) \brk*{\dlty{n_i}(1-\eta\beta) + \dltm{n_{i}+1}\sum_{k=n_{i}+1}^{n_{i+1}}\prod_{t=n_i+1}^{k-1} \gamma_{t+1}} \\
    &\leq \abs{\dlty{n_i}(1-\eta\beta)} + \sigma(\dlty{n_{i+1}}) \brk*{\dltm{n_{i}+1}\sum_{k=n_{i}+1}^{n_{i+1}}\prod_{t=n_i+1}^{k-1} \gamma_{t+1}},
\end{align*}
hence, since $0 \leq \eta\beta \leq 1$,
\begin{align*}
    \sigma(\dlty{n_{i+1}}) \brk*{\dltm{n_{i}+1}\sum_{k=n_{i}+1}^{n_{i+1}}\prod_{t=n_i+1}^{k-1} \gamma_{t+1}}
    &\geq (2+\eta\beta) \abs{\dlty{n_{i}}}
    \geq 2 \abs{\dlty{n_{i}}} \geq 0 \\
    \implies \sigma(\dlty{n_{i+1}}) = \sigma(\dltm{n_{i}+1}),
\end{align*}
where the last equality follows from the fact that for $t\geq 2, \gamma_t > 0$.
Thus, for $t>n_{i+1}$, again using \cref{eq:dystep} recursively,
\begin{align*}
    \dlty{t}
    &= \dlty{t-1} - \eta\dltf{t-1} + \dltm{t} \tag{\cref{eq:dystep}} \\
    &= \dlty{n_{i+1}} - \eta \sum_{k=n_{i+1}}^{t-1} \dltf{k} + \sum_{k=n_{i+1}+1}^{t} \dltm{k} \tag{\cref{eq:dystep} multiple times} \\
    &= \dlty{n_{i+1}} + \sum_{k=n_{i+1}+1}^{t} \dltm{k} \tag{\cref{lemma:dltfm}, $k>n_i \implies \dltf{k}=0$}.
\end{align*}
From \cref{lemma:dltfm}, we know that for all $k > n_{i}+1$,
\begin{align*}
    \sigma(\dltm{k})=\sigma(\dltm{k-1})=\dots=\sigma(\dltm{n_{i}+1}).
\end{align*}
Hence, since also $\sigma(\dlty{n_{i+1}})=\sigma(\dltm{n_{i}+1})$,
\begin{align*}
    \abs{\dlty{t}}
    &= \abs{\dlty{n_{i+1}} + \sum_{k=n_{i}+1}^{n_{i+1}} \dltm{k}} \\
    &= \abs{\dlty{n_{i+1}}} + \sum_{k=n_{i}+1}^{n_{i+1}}\abs{ \dltm{k}} \tag{all terms share sign} \\
    &\geq \abs{\dlty{n_{i+1}}},
\end{align*}
and we conclude the proof.
\end{proof}

\subsection{Proof of \cref{lemma:finite_m}}
\begin{proof}
First we will show that $M \geq 1$.
We already know that $\max_x \nabla f_0(x)=-G < -\frac12 G$.
From the definition of $f_1(x)$, $\max_x \nabla f_1(x)=-G+\beta(\ymax{n_1}-\ymin{n_1})$.
Let $(x_{t},y_{t})=\agd(f_0,0,t,\eta)$ and $(\xt_{t},\yt_{t})=\agd(f_0,\epsilon,0,\eta)$ for $t \leq n_1$.
Hence,
\begin{align*}
    \ymax{n_1}-\ymin{n_1}
    &= \abs{\dlty{n_1}} \tag{definition of $\ymax{n_1},\ymin{n_1}$} \\
    &= \abs{\dlty{n_1-1} - \eta \dltf{n_1-1} + \dltm{n_1}} \tag{\cref{eq:dystep}} \\
    &= \abs{\dlty{n_1-1}} \tag{\cref{lemma:dltfm}} \\
    &= \abs{\dlty{0}} \tag{repeating the two steps above} \\
    &= \epsilon < \frac{G}{2\beta}. \tag{assumption about $\epsilon$}
\end{align*}
Therefore,
\begin{align*}
    \max_x \nabla f_1(x)
    &< -G+\frac12 G
    = -\frac12 G,
\end{align*}
and $M \geq 1$.
Now we move to upper-bounding $M$.
Fix some finite $I \leq M$. From the definition of $M$ we know that $\max_x \nabla f_I(x) < -\frac{G}{2}$. Hence,
\begin{align*}
    -\frac{G}{2}
    &> \max_x \nabla f_{I}(x) 
    \\
    &= \max_x \left\{-G + \beta \int_{-\infty}^x \IND{\exists j \in [I] \text{ s.t. } z \in [\ymin{n_j},\ymax{n_j}]} dz \right\} 
    \\
    &\geq -G + \beta \brk{ \ymax{n_I}-\ymin{n_I} }
    ,
\end{align*}
which implies that $\ymax{n_I}-\ymin{n_I} < \frac{G}{2\beta}$.
But from \cref{lemma:exp_growth}, $\ymax{n_I}-\ymin{n_I} \geq \epsilon 3^{I-1}$, thus, $I \leq \log_3 \frac{3 G}{2\beta\epsilon}$ and since we fixed an arbitrary $I \leq M$, $M \leq \log_3 \frac{3 G}{2\beta\epsilon} \leq \ln \frac{3 G}{2\beta\epsilon}$.
In order to lower bound $\ymax{n_{M+1}}-\ymin{n_{M+1}}$ we will exploit the definition of $M$. We know that $\max_x \nabla f_{M+1}(x) \geq -\frac{G}{2}$ (this is the first function that violates the condition in the definition of $M$). Thus,
\begin{align*}
    -\frac{G}{2}
    &\leq \max_x \nabla f_{M+1}(x) 
    \\
    &= \max_x \left\{-G + \beta \int_{-\infty}^x \IND{\exists j \in [M+1] \text{ s.t. } z \in [\ymin{n_j},\ymax{n_j}]} dz \right\} 
    \\
    &\leq -G + \beta\sum_{j=1}^{M+1} \brk{ \ymax{n_j}-\ymin{n_j} }
    ,
\end{align*}
which implies, using \cref{lemma:exp_growth}, that
\begin{align*}
     \sum_{j=1}^{M+1} \abs{\dlty{n_j}} 
     =
     \sum_{j=1}^{M+1} \brk{ \ymax{n_j}-\ymin{n_j} }
     \geq 
     \frac{G}{2\beta}.
\end{align*}
On the other hand, since $\abs{\dlty{n_{j+1}}} \geq 3 \abs{\dlty{n_j}}$ due to \cref{lemma:ymy}, we have
\begin{align*}
    \frac{G}{2\beta}
    &\leq \sum_{j=1}^{M+1} \abs{\dlty{n_j}}
    \leq \sum_{j=1}^{M+1} \abs{\dlty{n_{M+1}}}\cdot 3^{j-(M+1)}
    \leq \frac{3}{2}\abs{\dlty{n_{M+1}}} \\
    &\implies \abs{\dlty{n_{I+1}}}
    = \abs{\dlty{n_{M+1}}}
    \geq \frac{G}{3\beta}.
    \qedhere
\end{align*}
\end{proof}

\subsection{Proof of \cref{lemma:f_props}}
\begin{proof}
The first derivative of $\fmp$ is
\begin{align*}
    \nabla \fmp(x) &\eqdef
    \begin{cases}
    \nabla f_M(x) &\qquad x \leq \qcp; \\
    \nabla f_M(\qcp)+\beta (x-\qcp) &\qquad \qcp < x \leq \qcp - \frac{1}{\beta}\nabla f_M(\qcp); \\
    0 &\qquad \text{otherwise}.
    \end{cases}
\end{align*}
Hence, for all $x \in \R$, $\abs{\nabla \fmp(x)} \leq \max\brk[c]{\abs{\nabla f_M(x)},\abs{\nabla f_M(\qcp)}} \leq G$, where the last inequality comes from \cref{lemma:f_i}. Thus, $\fmp$ is $G$-Lipschitz. The second derivative of $\fmp$ is
\begin{align*}
    \nabla^2 \fmp(x) &\eqdef
    \begin{cases}
    \nabla^2 f_M(x) &\qquad x \leq p; \\
    \beta &\qquad p < x \leq p - \frac{1}{\beta}\nabla f_M(\qcp); \\
    0 &\qquad \text{otherwise}.
    \end{cases}
\end{align*}
Thus, from \cref{lemma:f_i}, for all $x \in \R$, $0 \leq \nabla^2 \fmp(x) \leq \beta$. Hence, $\fmp$ is convex and $\beta$-smooth.
Let $x^{\star} \eqdef p-\frac{\nabla f_M(p)}{\beta}$.
Since $f$ is convex and $\nabla f(x^{\star})=0$,
$x^{\star} \in \argmin_x f(x)$.
In order to bound the distance $\abs{x_0-x^{\star}}$ we will first bound $p=\ymax{n_M}$.

Since $M \leq \ln \frac{3G}{2\beta\epsilon}$ (\cref{lemma:finite_m}),
\begin{align*}
    n_M \leq \ceil*{\frac{10}{\eta\beta}}\brk*{\ln \frac{3G}{2\beta\epsilon}+2}.
\end{align*}
Since for \agd, $\abs{m_t} \leq \abs{m_{t-1}}+\eta\abs{\nabla f(y_{t-1})}$ (\cref{eq:mstep}), for a $G$-Lipschitz function, $\abs{m_t} \leq \eta G (t-1)$.
Hence, for a sequence $(x_t,y_t)_{t=0}^T$ of \agd on a $G$-Lipschitz function,
\begin{align*}
    \abs{y_t}
    &\leq \abs{y_{t-1}} + \eta G (t-1) + \eta G \tag{\cref{eq:ystep} and Lipschitz condition} \\
    &= \abs{y_{t-1}} + \eta G t \\
    &= \abs{y_0} + \eta G \frac{t(t+1)}{2}.
\end{align*}
Therefore,
\begin{align*}
    \ymax{n_M} < \epsilon + \frac{\eta G}{2} \brk*{\ceil*{\frac{10}{\eta\beta}}\brk*{\ln \frac{3G}{2\beta\epsilon}+2}+1}^2.
\end{align*}
Hence,
\begin{align*}
    \abs{x_0-x^{\star}}
    &\leq \abs{x_0} + \abs{p} + \abs*{\frac{\nabla f_M(p)}{\beta}} \\
    &\leq \abs{x_0} + \frac{G}{\beta} + \abs{p} \tag{Lipschitz} \\
    &\leq \epsilon + \frac{G}{\beta} + \abs{p} \tag{$x_0\in\set{0,\epsilon}$} \\
    &< \epsilon + \frac{G}{\beta} + \epsilon + \frac{\eta G}{2} \brk*{\ceil*{\frac{10}{\eta\beta}}\brk!{\ln \frac{3G}{2\beta\epsilon}+2}+1}^2 \tag{showed above} \\
    &= O\brk2{\frac{G}{\eta\beta^2}\ln\brk2{\frac{G}{\beta \epsilon}}^2},
\end{align*}
where the last transition is due to $\epsilon < \frac{G}{2\beta}$ and $\eta \leq \frac{1}{\beta}$.
\end{proof}

\subsection{Proof of \cref{lemma:diff_equality}}
In order to prove the lemma, we will use the following lemma which tie the behaviour of \agd on $\fmp$ and $f_M$ for the first $n_M+1$ steps given our starting points. The lemma is proved in the following section.

\begin{lemma} \label{lemma:consistency_f_fm}
For all $x_0 \in \brk[c]{0,\epsilon}$ and $t \leq n_{M}+1$, we have 
$
    \agd(\fmp,x_0,t) = \agd(f_{M},x_0,t)
    .
$
\end{lemma}

Secondly we will need the following claim, proved below.
\begin{claim}\label{clm:post_plateau}
Let $(x_t,y_t) = \agd(\fmp,x_0,t,\eta)$ where $x_0 \in \brk[c]*{0,\epsilon}$. Then for all $t > n_{M}$, $\nabla \fmp(y_t)=0$.
\end{claim}

\begin{proof}[of \cref{lemma:diff_equality}]
For $t \leq n_M + 1$ the lemma follows from \cref{lemma:consistency_f_fm}.
We will show by induction that the lemma also follows for $t > n_M + 1$.
\begin{align*}
    x_t - \xt_t
    &= y_{t-1} - \eta \nabla f_M(y_{t-1}) - \yt_{t-1} + \eta \nabla f_M(\yt_{t-1}) \tag{\cref{eq:xstep}} \\
    &= y_{t-1} - \yt_{t-1} \tag{\cref{lemma:dltfm}} \\
    &= \yh_{t-1} - \yb_{t-1} \tag{induction assumption} \\
    &= \yh_{t-1} -\eta\nabla \fmp(\yh_{t-1}) - \yb_{t-1} + \eta\nabla \fmp(\yb_{t-1}) \tag{\cref{clm:post_plateau}} \\
    &= \xh_t - \xb_t. \tag{\cref{eq:xstep}}
\end{align*}
The second equality follows immediately from the first with \cref{eq:xstep}.
\end{proof}

In order to prove \cref{clm:post_plateau} we use the following claim.
\begin{claim} \label{clm:neg_der}
Let $f: \R \rightarrow \R$ be a convex, $\beta$-smooth function such that for all $x \in R$, $\nabla f(x) \leq 0$ ($\nabla f(x) < 0$). Let $(x_t,y_t)=\agd(f,x_0,t,\eta)$ for all $t \leq T$, for some $T$ and step size $\eta$. Then $y_t \geq y_{t-1}$ ($y_t > y_{t-1}$).
\end{claim}

\begin{proof}[of \cref{clm:neg_der}]
First we will show by induction that for all $1 \leq t \leq T$, $m_t \geq 0$. For $t=1$, $m_1=\gamma_1(x_1-x_0)=0\cdot(x_1-x_0)=0$. From \cref{eq:mstep}, $1 < t \leq T$, $m_t=\gamma_t(m_{t-1}-\eta \nabla f(y_{t-1}))$, and since $m_{t-1}$ is non-negative and $\nabla f(y_{t-1})$ is non-positive, $m_t \geq 0$.
In order to finish the claim we use \cref{eq:ystep},
\begin{align*}
    y_t
    &= y_{t-1} - \eta \nabla f(y_{y-1}) + m_t
    \geq y_{t-1} - \eta \nabla  f(y_{y-1}),
\end{align*}
where the last inequality follows from $m_t$ being non-negative. The claim follows in the respective case according to $\nabla f(x)$ being negative or non-positive.
\end{proof}

\begin{proof}[of \cref{clm:post_plateau}]
We will prove the claim by showing that $y_t > \qcp - \frac{1}{\beta}\nabla f_M(\qcp)$. Using \cref{clm:neg_der} it is enough to show that $y_{n_{M}+1} > \qcp - \frac{1}{\beta}\nabla f_M(\qcp)$.
\begin{align*}
    \qcp - \frac{1}{\beta}\nabla f_M(\qcp)
    &\leq \qcp + \frac{G}{\beta} \tag{Lipschitz condition of $f_M$} \\
    &\leq \ymin{n_{M}} + \frac{G}{2\beta} + \frac{G}{\beta} \tag{$p=\ymax{n_M}$ and \cref{clm:upper_bound_diff}} \\
    &\leq \ymin{n_{M}} + \frac{3G}{2\beta} \\
    &\leq y_{n_{M}} + \frac{3G}{2\beta} \\
    &\leq y_{n_{M}+1}. \tag{\cref{clm:min_step} with \cref{lemma:consistency_f_fm}}
\end{align*}
Hence, $\nabla f(y_t)=0$
\end{proof}

\subsection{Proof of \cref{lemma:consistency_f_fm}}

\begin{proof}[of \cref{lemma:consistency_f_fm}]
Let $(x_t,y_t) = \agd(f_{M},x_0,t)$ and $(\xt_t,\yt_t) = \agd(\fmp,x_0,t)$ for $t \leq n_{M}$. From \cref{clm:neg_der} we know that for all $t \leq n_{M}$,
\begin{align*}
    y_t \leq y_{n_{M}} \leq \ymax{n_{M}} = \qcp \implies \nabla f_M(y_t)=\nabla \fmp(y_t).
\end{align*}
We will continue by induction. At initialization, $y_0=x_0=\xt_0=\yt_0$. At step $t \leq n_{M}+1$,
\begin{align*}
    x_t
    &= y_{t-1} - \eta \nabla f_M(y_{t-1}) \\
    &= y_{t-1} - \eta \nabla \fmp(y_{t-1}) \tag{showed above} \\
    &= \yt_{t-1} - \eta \nabla \fmp(\yt_{t-1}) \tag{induction assumption} \\
    &= \xt_t. \tag{\cref{eq:xstep}}
\end{align*}
And using the induction assumption again, $y_t=\yt_t$, using \cref{eq:ystep}.
\end{proof}

\subsection{Proofs of Technical Claims} \label{appendix:technical}

Here we prove \cref{clm:cool_sum,clm:cool_sum2}.

\begin{proof}[of \cref{clm:cool_sum}]
We observe that
\begin{align*}
    (*) \eqdef \sum_{k=n}^{m}\prod_{t=n}^{k-1}\frac{t}{t+3}
    &= \brk*{1+\frac{n}{n+3}+\frac{n(n+1)}{(n+3)(n+4)} + \sum_{k=n+3}^{m} \frac{n(n+1)(n+2)}{k(k+1)(k+2)}} \\
    &\geq \brk*{1+\frac{n^3}{(n+1)^3}+\frac{n^3}{(n+2)^3} + \sum_{k=n+3}^{m} \frac{n^3}{k^3}} \\
    &= \sum_{k=n}^{m} \brk*{\frac{n}{k}}^3.
\end{align*}
We will lower bound the latter by integration,
\begin{align*}
    (*)
    &\geq \sum_{k=n}^{m} \brk*{\frac{n}{k}}^3 \\
    &\geq n^3 \int_{n}^{m+1} x^{-3}dx \\
    &= \frac{n^3}{2} \brk*{\frac{1}{n^2}-\frac{1}{(m+1)^2}} \\
    &= \frac{n}{2} \brk*{1-\frac{n^2}{(m+1)^2}}.
\end{align*}
\end{proof}

\begin{proof}[of \cref{clm:cool_sum2}]
In the case of $t=1$ both sides are $0$. For $t>1$,
\begin{align*}
    \sum_{k=2}^{t}\prod_{j=k}^{t} \gamma_j
    &= \sum_{k=2}^{t}\prod_{j=k}^{t}\frac{j-1}{j+2} \\
    &= \frac{t-1}{t+2} + \frac{(t-2)(t-1)}{(t+1)(t+2)} + \sum_{k=2}^{t-2}\frac{(k-1)k(k+1)}{t(t+1)(t+2)} \\
    &\geq \frac{(t-1)^3}{t^2(t+2)} + \frac{(t-2)^3}{t^2(t+2)} + \sum_{k=2}^{t-2}\frac{(k-1)^3}{t^2(t+2)} \\
    &= \frac{1}{t^2(t+2)}\sum_{k=2}^{t} (k-1)^3 \\
    &= \frac{1}{t^2(t+2)}\sum_{k=1}^{t-1} k^3.
\end{align*}
Using the formula for $\sum_{k=1}^{n} k^3=\frac{n^2(n+1)^2}{4}$,
\begin{align*}
    \sum_{k=2}^{t}\prod_{j=k}^{t} \gamma_j
    &\geq \frac{(t-1)^2 t^2}{4(t+2) t^2}
    = \frac{(t-1)^2}{4(t+2)}.
\end{align*}
\end{proof}

\section{Proofs for \cref{sec:uniform}}
\label{sec:uniform-proofs}

Here we prove \cref{lemma:reduction} and \cref{thm:main_as}.

\subsection{Proof of \cref{lemma:reduction}}
First we will need the following simple claim.
\begin{claim} \label{clm:nabla0eps}
Given $f_M$ and $\fmp$ from the construction in \cref{sec:initialization} with parameters $G,\beta,\eta,\epsilon$,
\begin{align*}
    \nabla f_M(0) = \nabla f_M(\epsilon) = \nabla \fmp(0) = \nabla \fmp(\epsilon) = -G.
\end{align*}
\end{claim}

\begin{proof}
Let $(x_t,y_t)=\agd(f_M,0,t)$ and $(\xt_t,\yt_t)=\agd(f_M,\epsilon,t)$ for $t \leq T$.
From \cref{clm:eta_t_ni}, for all $i \in [M]$, $y_0 < \ymin{n_i}$ and $\yt_0 < \ymin{n_i}$.
Thus,
\begin{align*}
    \nabla f_M(0)
    &= \nabla f_M(y_0)
    = -G + \beta\int_{-\infty}^{y_0} \I\brk[s]{\exists j \in [M] \text{ s.t. } z \in [y_{n_j}^{min},y_{n_j}^{max}]} dz = -G,
\end{align*}
and similarly, $\nabla f_M(\epsilon)=\nabla f_M(\yt_0) = -G$.
Since $p=\ymax{n_M}>\max\brk[c]{y_0,\yt_0}$ as we stated above, and from the definition of $\fmp$, for all $x \in (-\infty,p]$, $\nabla \fmp(x)=\nabla f_M(x)$, then also $\nabla \fmp(0)=\nabla \fmp(\epsilon)=-G$.
\end{proof}

\begin{proof}[of \cref{lemma:reduction}]
Both parts are similar. We start with the first equality and prove by induction.
Let $(x_t,y_t)=\agd(\fmp,0,t,\eta)$ and $(\xt_t,\yt_t)=\agd(R_{S},0,t,\hat{\eta})$ for $t \leq T$.
For $t=1$,
\begin{align*}
    \xt_1
    &= \yt_0 - \hat{\eta}\nabla R_{S}(\yt_0) \tag{\cref{eq:xstep}} \\
    &= y_0 - \hat{\eta} \frac{n-3}{n} \nabla \fmp(y_0) \tag{$\yt_0=0=y_0$ and \cref{eq:rs}} \\
    &= y_0 - \eta \nabla \fmp(0) \tag{def of $\eta$, $\eta=\hat{\eta}\frac{n-3}{n}$} \\
    &= x_1. \tag{\cref{eq:xstep}}
\end{align*}
Since $\gamma_1=0$,
\begin{align*}
    \yt_1
    &= \xt_1 + \gamma_1 (\xt_1-\xt_0) = \xt_1,
\end{align*}
and similarly $x_1=y_1$, hence $\yt_1=y_1$.
For $t>1$,
\begin{align*}
    \xt_{t}
    &= \yt_{t-1} - \hat{\eta} \nabla R_S(\yt_{t-1}) \tag{\cref{eq:xstep}} \\
    &= \yt_{t-1} - \hat{\eta} \frac{n-3}{n}\nabla \fmp(\yt_{t-1}) \tag{\cref{eq:rs}} \\
    &= \yt_{t-1} - \eta \nabla \fmp(\yt_{t-1}) \tag{def of $\eta$, $\eta=\hat{\eta}\frac{n-3}{n}$} \\
    &= y_{t-1}-\eta\nabla h(y_{t-1}) \tag{induction assumption} \\
    &= x_t \tag{\cref{eq:xstep}}.
\end{align*}
We conclude with
\begin{align*}
    y_t
    &= x_t + \gamma_t (x_t-x_{t-1}) \tag{\cref{eq:ystep}} \\
    &= \xt_t + \gamma_t (\xt_t-x_{t-1}) \tag{from above} \\
    &= \xt_t + \gamma_t (\xt_t-\xt_{t-1}) \tag{induction assumption} \\
    &= \yt_t. \tag{\cref{eq:ystep}}
\end{align*}
Now we move to prove the second equality, again by induction.
Let $(x_t,y_t)=\agd(\fmp,\epsilon,t,\eta)$ and $(\xt_t,\yt_t)=\agd(R_{S'},0,t,\hat{\eta})$ for $t \leq T$.
For $t=1$,
\begin{align*}
    \xt_1
    &= \yt_0 - \hat{\eta}\nabla R_{S'}(y_0) \tag{\cref{eq:xstep}} \\
    &= 0 - \hat{\eta} \frac{1}{n} \nabla g_2(0) - \hat{\eta} \frac{n-3}{n} \nabla \fmp(0) \tag{$\yt_0=0$ and \cref{eq:rst}} \\
    &= -\eta \frac{1}{n-3} \nabla g_2(0) - \eta \nabla \fmp(0) \tag{def of $\eta$, $\eta=\hat{\eta}\frac{n-3}{n}$} \\
    &= \frac{\beta\eta^2 G}{n-3} - \eta \nabla \fmp(0) \tag{$\nabla g_2(x)=-\beta\eta G + \beta\cdot \int_{-\infty}^{x} \IND{z \in \brk[s]*{0,\eta G}} dz$} \\
    &= \epsilon - \eta \nabla \fmp(0) \tag{$\epsilon=\frac{\beta\eta^2 G}{n-3}$} \\
    &= \epsilon - \eta \nabla \fmp(\epsilon) \tag{\cref{clm:nabla0eps}} \\
    &= x_1. \tag{\cref{eq:xstep} and $y_0=\epsilon$}
\end{align*}
Since $\gamma_1=0$,
\begin{align*}
    \yt_1
    &= \xt_1 + \gamma_1 (\xt_1-\xt_0) = \xt_1,
\end{align*}
and similarly $x_1=y_1$, hence $\yt_1=y_1$.
Now for $t>1$.
Repeating the steps we did for $(1)$,
\begin{align*}
    \xt_{t}
    &= \yt_{t-1} - \hat{\eta} \nabla R_{S'}(\yt_{t-1}) \tag{\cref{eq:xstep}} \\
    &= \yt_{t-1} - \hat{\eta} \frac{n-3}{n}\nabla \fmp(\yt_{t-1}) - \hat{\eta} \frac{1}{n}\nabla g_2(\yt_{t-1}) \tag{\cref{eq:rst}} \\
    &= \yt_{t-1} - \eta \nabla \fmp(\yt_{t-1}) - \hat{\eta} \frac{1}{n}\nabla g_2(\yt_{t-1}) \tag{def of $\eta$, $\eta=\hat{\eta}\frac{n-3}{n}$} \\
    &= y_{t-1}-\eta\nabla \fmp(y_{t-1}) - \hat{\eta} \frac{1}{n}\nabla g_2(y_{t-1}) \tag{induction assumption} \\
    &= x_t - \hat{\eta} \frac{1}{n}\nabla g_2(y_{t-1}) \tag{\cref{eq:xstep}}.
\end{align*}
We need to show that $\nabla g_2(y_{t-1})=0$.
Since
\begin{align*}
    \nabla g_2(y_{t-1})
    &= -\beta\eta G + \beta\cdot \int_{-\infty}^{y_{t-1}} \IND{z \in \brk[s]*{0,\eta G}} dz,
\end{align*}
it is enough to show that $y_{t-1}\geq\eta G$.
From \cref{clm:neg_der}, for $t>1$, $y_{t-1} \geq y_1$.
Thus,
\begin{align*}
    y_{t-1}
    &\geq y_1 \\
    &= x_1 + \underset{=0}{\gamma_1}(x_1-x_0) \tag{\cref{eq:ystep}} \\
    &= y_0 - \eta \nabla \fmp(y_0) \tag{\cref{eq:xstep}} \\
    &= \epsilon + \eta G. \tag{\cref{clm:nabla0eps}}
\end{align*}
So indeed $y_{t-1} \geq \eta G$, hence, $\xt_t=x_t$.
We conclude with
\begin{align*}
    y_t
    &= x_t + \gamma_t (x_t-x_{t-1}) \tag{\cref{eq:ystep}} \\
    &= \xt_t + \gamma_t (\xt_t-x_{t-1}) \tag{from above} \\
    &= \xt_t + \gamma_t (\xt_t-\xt_{t-1}) \tag{induction assumption} \\
    &= \yt_t. \tag{\cref{eq:ystep}}
\end{align*}
\end{proof}

\subsection{Proof of \cref{thm:main_as}}
\begin{proof}
Let $c_3=c_1/4$ and $c_4=c_2/4$. In order to lower bound the uniform stability of algorithm $A$ we need to pick $S,S'\in \Z^n$ and $z \in \Z$ and lower bound
\begin{aligni*}
    \abs{\ell(A(S),z)-\ell(A(S'),z)}.
\end{aligni*}
We use $S,S'$ which we defined at \cref{sec:uniform}, with $z=3 \in \Z$.
We showed at \cref{lemma:reduction} that we match the iterations on $\fmp$ of our construction.
For $z=3$, since $f(w;3)=-Gw$,
\begin{align*}
    \abs{\ell(x_T,z)-\ell(\xt_T,z)} = G \abs{x_T-\xt_T}.
\end{align*}
The first case is when $T \in \brk[c]!{\ceil{10n/\hat{\eta}\hat{\beta}(n-3)}(i+2) ~:~ i = 1,2,3 \dots}$, for which
\begin{align*}
    T \in \brk[c]!{\ceil{10/\eta\beta}(i+2) ~:~ i = 1,2,3 \dots}. \tag{$\eta=\frac{n-3}{n}\hat{\eta}$,$\beta=\hat{\beta}$}
\end{align*}
So using \cref{thm:main_init_stab}, based on $S,S',z$ and our definitions for $G,\beta,\eta \text{ and } \epsilon$,
\begin{align*}
    \unif{\agd_T}{\ell}(n)
    &\geq G \min\brk[c]1{\tfrac{G}{3\beta}, c_2 e^{c_1 \eta\beta T}\epsilon}
    \geq\min\brk[c]1{\tfrac{\hat{G}^2}{3\hat{\beta}}, c_4 e^{c_3\hat{\eta}\hat{\beta} T} \tfrac{\hat{\beta} \hat{\eta}^2 \hat{G}^2}{n}}.
\end{align*}
Above we used the fact that $\eta = \frac{n-3}{n} \hat{\eta} \geq \frac{\hat{\eta}}{4}$ since $n \geq 4$.
Now for the second case of $T>\ceil{40/\hat{\eta}\hat{\beta}}\brk*{\ln \frac{6n}{\hat{\eta}^2\hat{\beta}^2}+3}$, in which,
\begin{align*}
    T
    &> \ceil{40/\hat{\eta}\hat{\beta}}\brk*{\ln \frac{6n}{\hat{\eta}^2\hat{\beta}^2}+3} \\
    &\geq \ceil{10n/\hat{\eta}\hat{\beta}(n-3)}\brk*{\ln \frac{3n^2}{2\hat{\eta}^2\hat{\beta}^2(n-3)}+3} \tag{$n \geq 4$} \\
    &= \ceil{10/\eta\beta}\brk*{\ln \frac{3(n-3)}{2\eta^2\beta^2}+3} \tag{$\eta=\frac{n-3}{n}\hat{\eta}$,$\beta=\hat{\beta}$} \\
    &= \ceil{10/\eta\beta}\brk*{\ln \frac{3G}{2\beta\epsilon}+3} \tag{$\epsilon=\frac{\beta\eta^2 G}{n-3}$}.
\end{align*}
Hence, using \cref{thm:main_init_stab},
\begin{align*}
    \abs{\ell(x_T,z)-\ell(\xt_T,z)}
    &= \frac{G^2}{3\beta}
    = \frac{\hat{G}^2}{3\hat{\beta}}
    \implies \unif{\agd_T}{\ell}(n) \geq \frac{\hat{G}^2}{3\hat{\beta}}.
    \qedhere
\end{align*}
\end{proof}

\section{Other Variants of \agd} \label{appendix:variants}

Here we show the equivalence between the version of \agd we analyze in this paper, and other versions that are common in the literature.

\subsection{Variant I}

First we consider a common variant that appears for example in \citep{allen2017linear}.
Starting at $\zt_0 = \yt_0=\xt_0$, this variant proceeds for $t=1,2,\ldots$ as:
\begin{align}
    \xt_{t+1} &= \tau_t \zt_t + (1-\tau_t) \yt_t; \label{eq:xallen} \\
    \yt_{t+1} &= \xt_{t+1} - \tfrac{1}{\beta}\nabla f(\xt_{t+1}); \label{eq:yallen} \\
    \zt_{t+1} &= \zt_{t} - \alpha_{t+1}\nabla f(\xt_{t+1}), \label{eq:zallen}
\end{align}
where $\alpha_{t+1} = \frac{t+2}{2\beta}$ and $\tau_t=\frac{2}{t+2}$.

Our next claim establishes that this variant is precisely equivalent to the \agd iterations considered in the paper (\cref{eq:agd-xt,eq:agd-yt}) with step size $\eta=\ifrac{1}{\beta}$.

\begin{claim}
For all $t \in [T]$, $\xt_t=y_{t-1}$ and $\yt_t = x_t$.
\end{claim}

\begin{proof}
We will prove using induction. For $t=1$, $\xt_1=y_0$ follows from
\begin{align*}
    \xt_1
    &= \tau_0 \zt_0 + (1-\tau_0) \yt_0 \tag{\cref{eq:xallen}} \\
    &= \tau_0 \xt_0 + (1-\tau_0) \xt_0 \tag{initialization} \\
    &= \xt_0 \\
    &= x_0 \tag{initialization} \\
    &= y_0. \tag{initialization}
\end{align*}
The second equality, $\yt_t=x_t$, is immediate for all $t \in [T]$ given $\xt_t=y_{t-1}$ since
\begin{align*}
    \yt_t
    &= \xt_t - \tfrac{1}{\beta}\nabla f(\xt_t) \tag{\cref{eq:yallen}} \\
    &= y_{t-1} - \tfrac{1}{\beta}\nabla f(y_{t-1}) \tag{from $(1)$} \\
    &= x_t. \tag{\cref{eq:xstep}}
\end{align*}
We finish by showing that $\xt_{t+1}=y_t$,
\begin{align*}
    \xt_{t+1}
    &= \yt_t + \tau_t (\zt_t-\yt_t) \tag{\cref{eq:xallen}} \\
    &= \yt_t + \frac{2}{t+2} (\zt_t-\yt_t) \tag{$\tau_t=\frac{2}{t+2}$} \\
    &= \yt_t + \gamma_t \frac{2}{t-1} (\zt_t-\yt_t) \tag{$\gamma_t=\frac{t-1}{t+2}$} \\
    &= \yt_t + \gamma_t \brk3{\frac{2}{t-1} \zt_t-\frac{2}{t-1} \yt_t} \\
    &= \yt_t + \gamma_t \brk3{\frac{2}{t-1} \zt_t+\yt_t-\frac{t+1}{t-1} \yt_t} \\
    &= \yt_t + \gamma_t \brk3{\yt_t + \frac{2}{t-1} \zt_{t-1}-\frac{t+1}{t-1} \xt_t + \underbrace{\brk3{\frac{t+1}{\beta(t-1)}-\frac{2\alpha_t}{t-1}}}_{=0} \nabla f(\xt_t)} \tag{\cref{eq:yallen},\cref{eq:zallen}}
    \\
    &= \yt_t + \gamma_t \brk3{\yt_t + \frac{2}{t-1} \zt_{t-1}-\frac{t+1}{t-1} ((1-\tau_{t-1})\yt_{t-1}+\tau_{t-1}\zt_{t-1})} \tag{\cref{eq:xallen}} \\
    &= \yt_t + \gamma_t \brk3{\yt_t -\underbrace{\frac{(t+1)(1-\tau_{t-1})}{t-1}}_{=1} \yt_{t-1} + \underbrace{\brk*{\frac{2}{t-1}-\frac{(t+1)\tau_{t-1}}{t-1}}}_{=0} \zt_{t-1}} \\
    &= \yt_t + \gamma_t (\yt_t - \yt_{t-1}) \\
    &= x_t + \gamma_t (x_t - x_{t-1}) \tag{induction assumption} \\
    &= y_t. \tag{\cref{eq:ystep}}
\end{align*}
\end{proof}

\subsection{Variant II}

\newcommand{\xag}[0]{\tilde{x}^{\text{ag}}}
\newcommand{\xmd}[0]{\tilde{x}^{\text{md}}}
\newcommand{\btild}[0]{\tilde{\beta}}
\newcommand{\gtild}[0]{\tilde{\gamma}}

Next, we consider a second variant of \agd that appears in, e.g., \citet{lan2012optimal}. 
Starting at $\xag_1 = \xt_1$, this version takes the form
\begin{align}
    \xmd_{t} &= \btild_t^{-1} \xt_t + (1-\btild_t^{-1}) \xag_t; \label{eq:xmdcom} \\
    \xt_{t+1} &= \xt_{t} - \gtild_t\nabla f(\xmd_{t+1}); \label{eq:xcom} \\
    \xag_{t+1} &= \btild_t^{-1} \xt_{t+1} + (1-\btild_t^{-1}) \xag_t, \label{eq:xagcom}
\end{align}
where $\btild_t = (t+1)/2$ and $\gtild_t = (t+1)/4\beta$.

The claim below establishes the equivalence between this variant and the version of \agd given in \cref{eq:agd-xt,eq:agd-yt} with step size $\eta=\ifrac{1}{2\beta}$. 

\begin{claim}
For all $t \in [T]$, $\xag_t\overset{(1)}{=}x_{t-1}$ and $\xmd_t \overset{(2)}{=} y_{t-1}$.
\end{claim}

\begin{proof}
We will prove by induction.
For $t=1$, the first equality follows from the initialization,
\begin{align*}
    \xag_1
    &= \xt_1
    = x_0.
\end{align*}
The second equality follows from
\begin{align*}
    \xmd_1
    &= \btild_1^{-1}\xt_1+(1-\btild_1^{-1})\xag_1 \tag{\cref{eq:xmdcom}} \\
    &= \btild_1^{-1}\xt_1+(1-\btild_1^{-1})\xt_1 \tag{initialization} \\
    &= \xt_1 \\
    &= x_0 \tag{initialization} \\
    &= y_0. \tag{initialization}
\end{align*}
For $t>1$, we first show that $\xag_t=x_{t-1}$,
\begin{align*}
    \xag_t
    &= \btild_{t-1}^{-1}\xt_{t}+(1-\btild_{t-1}^{-1})\xag_{t-1} \tag{\cref{eq:xagcom}} \\
    &= \btild_{t-1}^{-1}(\xt_{t-1}-\gtild_{t-1}\nabla f(\xmd_{t-1}))+(1-\btild_{t-1}^{-1})\xag_{t-1} \tag{\cref{eq:xcom}} \\
    &= \xmd_{t-1} -
    \btild^{-1}_{t-1}\gtild_{t-1}\nabla f(\xmd_{t-1}) \tag{\cref{eq:xmdcom}} \\
    &= \xmd_{t-1} -
    \eta\nabla f(\xmd_{t-1}) \tag{$\eta=\frac{1}{2\beta}=\btild_{t-1}^{-1}\gtild_{t-1}$} \\
    &= y_{t-2} -
    \eta\nabla f(y_{t-2}) \tag{induction assumption} \\
    &= x_{t-1}. \tag{\cref{eq:xstep}} 
\end{align*}
We conclude by showing that $\xmd_t=y_{t-1}$ using the induction assumption and the equality $\xag_t=x_{t-1}$ we showed above,
\begin{align*}
    \xmd_{t}
    &= \btild_t^{-1}\xt_t + (1-\btild_t^{-1}) \xag_t \tag{\cref{eq:xmdcom}} \\
    &= \frac{\btild_{t-1}}{\btild_t}\btild_{t-1}^{-1}\xt_t + (1-\btild_t^{-1}) \xag_t \\
    &= \frac{t}{t+1}\btild_{t-1}^{-1}\xt_t + (1-\btild_t^{-1}) \xag_t \tag{$\btild_t=\frac{t+1}{2}$} \\
    &= \frac{t}{t+1}(\xag_t+(1-\btild_{t-1}^{-1})\xag_{t-1}) + (1-\btild_t^{-1}) \xag_t \tag{\cref{eq:xagcom}} \\
    &= \xag_t \brk*{\frac{t}{t+1}+1-\btild_t^{-1}}+\frac{t(1-\btild_{t-1}^{-1})}{t+1}\xag_{t-1} \\
    &= \xag_t \brk*{\frac{t}{t+1}+1-\frac{2}{t+1}}+\frac{t-2}{t+1}\xag_{t-1} \tag{$\btild_t=\frac{t+1}{2}$} \\
    &= \xag_t + \gamma_{t-1}(\xag_t-\xag_{t-1}) \tag{$\gamma_{t-1}=\frac{t-2}{t+1}$} \\
    &= x_{t-1} + \gamma_{t-1}(x_{t-1}-x_{t-2}) \tag{induction assumption} \\
    &= y_{t-1}. \tag{\cref{eq:ystep}}
\end{align*}
\end{proof}

\section{Initialization Stability Upper Bounds} 
\label{sec:initi_stab}

In this section we prove initialization bounds for GD in the convex and smooth setting and NAG in the setting of a quadratic objective.

\subsection{Gradient Descent, Smooth Objectives}
\label{sec:upper-gd}

In this section we consider fixed step-size GD in the convex and $\beta$-smooth setting. The update rule of this version of GD is $x_{t+1}=x_t-\eta \nabla f(x_t)$, where $0 < \eta \leq \frac{1}{\beta}$.

\begin{claim} \label{clm:gd_init}
    Let $f$ be a convex, $\beta$-smooth function. Then for all $x_0 \in \R^d$, $\epsilon>0$, and $T \geq 1$,
\begin{align*}
    \init{\text{GD}_T}(x_0,\epsilon) 
    \leq 
    \epsilon.
\end{align*}
\end{claim}

This bound is tight since for $f \eqdef 0$, we have trivially that $\init{\text{GD}_T}(x_0,\epsilon)=\epsilon$ for all $T$.
The proof of the above claim mostly follow arguments of \cite{hardt2016train} and is given here for completeness.
First we state the well-known co-coercivity property of the gradient operator over smooth functions (e.g., \citealp{nesterov2003introductory}).

\begin{lemma} \label{lemma:coer}
Let $f$ be a convex and $\beta$-smooth function on $\R^d$. Then for any $u,v\in\R^d$, we have
\begin{align*}
    (\nabla f(u) - \nabla f(v))^T (u-v) \geq \frac{1}{\beta}\norm{\nabla f(u)-\nabla f(v)}_2^2.
\end{align*}
\end{lemma}

Below is a simple contractive property of GD based on this lemma.
\begin{corollary} \label{cor:contractive}
    Let $f$ be a convex and $\beta$-smooth function on $\R^d$. Then for any $u,v\in\R^d$ and $\eta \leq \frac{2}{\beta}$, we have
\begin{align*}
    \norm{ (u-\eta\nabla f(u)) - (v-\eta\nabla f(v))) }_2
    &\leq 
    \norm{u-v}_2
    .
\end{align*}
\end{corollary}

\begin{proof}
Write:
\begin{align*}
    &\norm{ (u-\eta\nabla f(u)) - (v-\eta\nabla f(v))) }_2^2 \\
    &\qquad= 
    \norm{u-v}_2^2 + \eta^2\norm{\nabla f(u)-\nabla f(v))}_2^2 - 2\eta(u-v)^T(\nabla f(u)-\nabla f(v))) \\
    &\qquad\leq 
    \norm{u-v}_2^2 + \eta^2\norm{\nabla f(u)-\nabla f(v))}_2^2 - \frac{2\eta}{\beta}\norm{\nabla f(u)-\nabla f(v))}_2^2 \\
    &\qquad\leq 
    \norm{u-v}_2^2
    .\qedhere
\end{align*}
\end{proof}

We can now prove our claim.
\begin{proof}[of \cref{clm:gd_init}]
Let $x_0,\xt_0$ be our starting points such that $\norm{x_0-\xt_0}\leq \epsilon$.
Let $(x_t)_{t=0}^{T-1}$ and $(\xt_t)_{t=0}^{T-1}$ be the iterations of GD over $f$ starting at $x_0$ and $\xt_0$ respectively.
Thus, by \cref{cor:contractive}, 
\begin{align*}
    \norm{x_t-\xt_t}
    = 
    \norm{(x_{t-1}-\eta\nabla f(x_{t-1})) - (\xt_{t-1}-\eta\nabla f(\xt_{t-1}))} 
    \leq \norm{x_{t-1}-\xt_{t-1}}.
\end{align*}
Invoking the same argument recursively,
\begin{align*}
    \norm{x_t-\xt_t}
    &\leq \norm{x_0-\xt_0}
    \leq \epsilon.
\end{align*}
Hence, the initialization stability of GD is at most $\epsilon$.
\end{proof}

\newcommand{\proj}[1]{\Pi_{\Omega}\brk[s]*{#1}}

\subsection{Gradient Descent, Non-smooth Objectives}
\label{sec:upper-gd-nonsmooth}

In this section we consider GD with a constant step size in the convex and non-smooth setting. The update rule of this version of GD is
\begin{align*}
    x_{t+1}=\proj{x_t-\eta \nabla f(x_t)},
\end{align*}
where $\proj{\cdot}$ is the Euclidean projection onto a compact convex set $\Omega \subseteq \R^d$. 
Often, the final output of the algorithm is the average of the iterates.
The claim below holds for both final and average versions.

\begin{claim}
Let $f$ be a convex, $G$-Lipschitz function. Then for GD with $T$ steps, for all $x_0 \in \R^d$ and $\epsilon>0$,
\begin{align*}
    \init{\text{GD}_T}(x_0,\epsilon) \leq \epsilon + 2G \eta \sqrt{T}.
\end{align*}
\end{claim}

The proof is similar to the one of \citet{bassily2020stability}; we give it here for completeness.

\begin{proof}
Let $\xt_0 \in \R^d$ such that $\norm{\xt_0-x_0}\leq\epsilon$.
Let $\delta_t \eqdef \norm{x_t - \xt_t}$.
Then,
\begin{align*}
    \delta_{t+1}^2
    &= \norm*{\proj{x_t-\eta\nabla f(x_t)}-\proj{\xt_t-\eta\nabla f(\xt_t)}}^2 \\
    &\leq \norm*{\brk*{x_t-\eta\nabla f(x_t)}-\brk*{\xt_t-\eta\nabla f(\xt_t)}}^2 \\
    &= \delta_t^2 + \eta^2\norm*{\nabla f(x_t)-\nabla f(\xt_t)}^2 -2 \eta \brk[a]*{\nabla f(x_t)-\nabla f(\xt_t),x_t-\xt_t} \\
    &\leq \delta_t^2 + \eta^2\norm*{\nabla f(x_t)-\nabla f(\xt_t)}^2,
\end{align*}
where the last inequality follows from convexity,
\begin{align*}
    f(x_t)
    &\geq f(\xt_t) + \brk[a]*{\nabla f(\xt_t),x_t-\xt_t} \\
    &\geq f(x_t) + \brk[a]*{\nabla f(x_t),\xt_t-x_t} + \brk[a]*{\nabla f(\xt_t),x_t-\xt_t} \\
    \implies &\brk[a]*{\nabla f(x_t)-\nabla f(\xt_t),x_t-\xt_t} \geq 0.
\end{align*}
From the Lipschitz condition,
\begin{align*}
    \norm*{\nabla f(x_t)-\nabla f(\xt_t)} \leq \norm*{\nabla f(x_t)}+\norm*{\nabla f(\xt_t)} \leq 2G,
\end{align*}
hence
\begin{align*}
    \delta_{t+1}^2 \leq \delta_{t}^2 + 4\eta^2 G^2.
\end{align*}
Invoking the argument above recursively,
\begin{align*}
    \delta_{t}
    &\leq \sqrt{\delta_0^2 + 4\eta^2 G^2 t} \\
    &\leq \delta_0 + 2 \eta G \sqrt{t} \\
    &\leq \epsilon + 2 \eta G \sqrt{t}.
\end{align*}
Since this bound holds for all $t=0,\dots,T-1$, the bound also holds after averaging,
\begin{align*}
    \norm*{\frac{1}{T} \sum_{t=0}^{T-1} x_t - \frac{1}{T} \sum_{t=0}^{T-1} \xt_t } 
    \leq 
    \epsilon + 2 \eta G \sqrt{T}
    .
\end{align*}
Thus we proved our initialization stability bound.
\end{proof}

We note that this bound is tight up to a constant factor: this can be shown using the same type of construction as \citet{bassily2020stability} use for lower bounding the uniform stability of GD in the non-smooth case.
The idea is to use initial points $x_0=0$, $\xt_0=(\ifrac{\epsilon}{\sqrt{d}}) \cdot (1,\dots,1)$, with the following objective function over the unit ball:
\begin{align*}
    f(x) = G \max\brk[c]{0,x_1-c,\dots,x_d-c},
\end{align*}
for $c < \ifrac{\epsilon}{\sqrt{d}}$.
The first trajectory will stay put as $x_1=0$ is a minimizer;
for the second trajectory, the arguments as in \citep{bassily2020stability} show that at iteration $i$ a valid sub-gradient will be $G e_i$ ($e_i$ is the $i$'th standard basis element), so that $\xt_t=\xt_0-\sum_{i=1}^{t} G \eta e_i$ and we will have
\begin{align*}
    \norm{x_t-\xt_t} = \norm{\xt_t} = \Omega(G \eta \sqrt{t}).
\end{align*}

\subsection{Accelerated Gradient Method, Convex Objectives} 
\label{sec:agd-upper}

In this section we provide an exponential upper bound of initialization stability for \agd in the convex and smooth setting.

\begin{claim} \label{clm:agd_quad_init_gen}
    Let $f$ be convex and $\beta$-smooth function. Then for \agd with step size $\eta\leq\ifrac{1}{\beta}$ and $T$ steps, for all $x_0 \in \R^d$ and $\epsilon>0$,
\begin{align*}
    \init{\agd_T}(x_0,\epsilon) \leq \epsilon+ \eta\beta\epsilon 3^{T-1}.
\end{align*}
\end{claim}

\begin{proof}
Let $x_0,\xt_0$ be our starting points s.t. $\norm{x_0-\xt_0}\leq \epsilon$.
Let us consider two runs of the method initialized at $x_0, \xt_0$ respectively:
\begin{align*}
    (x_t,y_t) &= \agd(f, x_0, t),
    \qquad
    (\xt_t,\yt_t) = \agd(f, \xt_0, t),
    \qquad\qquad
    \forall ~ t \geq 0,
\end{align*}
We will show the bound by first proving that $\norm{\dlty{t}} \leq \epsilon + \eta\beta\epsilon 3^{t}$ by induction.
For $t=0$ the claim is immediate. Assuming the claim is correct for $k=0,\dots,t$, we will prove for $t+1$.
\begin{align*}
    \norm{\dlty{t+1}}
    &= \norm{\dlty{t}-\eta \dltf{t}+\dltm{t+1}} \tag{\cref{eq:dystep}} \\
    &\leq \norm{\dlty{t}-\eta \dltf{t}} + \norm{\dltm{t+1}} \\
    &\leq \norm{\dlty{t}} + \norm{\dltm{t+1}} \tag{\cref{cor:contractive}} \\
    &\leq \norm{\dlty{t}} + \norm{\dltm{1}}+\eta\sum_{k=1}^{t}\norm{\dltf{k}} \tag{\cref{eq:dmstep} recursively with $\gamma_k \leq 1$} \\
    &= \norm{\dlty{t}} + \eta\sum_{k=1}^{t}\norm{\dltf{k}} \tag{$\dltm{1}=0$ since $\gamma_1=0$} \\
    &\leq \norm{\dlty{t}} + \eta\beta\sum_{k=1}^{t}\norm{\dlty{k}} \tag{smoothness} \\
    &\leq \epsilon+\eta\beta\epsilon 3^{t} + \eta\beta\sum_{k=1}^{t}(\epsilon+\eta\beta\epsilon 3^{k}) \tag{induction assumption} \\
    &\leq \epsilon+\eta\beta\epsilon 3^{t} + \eta\beta\sum_{k=1}^{t}(\epsilon+\epsilon 3^{k}) \tag{$\eta\leq\ifrac{1}{\beta}$} \\
    &= \epsilon+\eta\beta\epsilon 3^{t} + \eta\beta\epsilon\brk*{t+3\frac{3^{t}-1}{2}}
    \\
    &\leq \epsilon+ \eta\beta\epsilon\brk*{t+3^{t}+3 \frac{3^{t}}{2}} \\
    &\leq \epsilon+ \eta\beta\epsilon\brk*{\frac{3^{t}}{2}+3^{t}+3 \frac{3^{t}}{2}} \tag{$t \geq 0 \implies 3^{t}\geq 2t$} \\
    &= \epsilon+ \eta\beta\epsilon 3^{t+1}.
\end{align*}
We finish with
\begin{align*}
    \norm{\dltx{T}}
    &= \norm{\dlty{T-1}-\eta\dltf{T-1}} \tag{\cref{eq:dxstep}} \\
    &\leq \norm{\dlty{T-1}} \tag{\cref{cor:contractive}} \\
    &\leq \epsilon+ \eta\beta\epsilon 3^{T-1}.
    \qedhere
\end{align*}
\end{proof}

\subsection{Accelerated Gradient Method, Quadratic Objectives} 
\label{sec:agd_quad}

The argument we give below is similar to the one presented by \citet{chen2018stability} for bounding the uniform stability of \agd for quadratic objectives and relies on some of their technical results, which we state and prove here for completeness.
The following claim bound the initialization stability of \agd for a quadratic objective.
\begin{claim} \label{clm:agd_quad_init}
    Let $f$ be a quadratic function with a positive semi-definite Hessian. Then for \agd with $T$ steps, for all $x_0 \in \R^d$ and $\epsilon>0$,
\begin{align*}
    \init{\agd_T}(x_0,\epsilon) \leq 4T\epsilon.
\end{align*}
\end{claim}

In order to prove \cref{clm:agd_quad_init} we need the following technical claim (proof at \cref{sec:quad_multidim}).
\begin{lemma} \label{lemma:quad_multidim}
Suppose 
$M_k=\begin{pmatrix}
(1+\gamma_{k})A & -\gamma_{k}A \\
1 & 0
\end{pmatrix}$,
where $0 \preceq A \preceq 1$ and $-1 \leq \gamma_{k} \leq 1$. Then
\begin{align*}
    \norm3{\prod_{k=1}^{t}M_k} \leq 2(t+1).
\end{align*}
\end{lemma}

\begin{proof}[of \cref{clm:agd_quad_init}]
Let $0 \preceq H \preceq \beta$ be the Hessian of $f$.
Let $x_0,\xt_0$ be our starting points s.t. $\norm{x_0-\xt_0}\leq \epsilon$.
Let us consider two runs of the method initialized at $x_0, \xt_0$ respectively:
\begin{align*}
    (x_t,y_t) &= \agd(f, x_0, t),
    \qquad
    (\xt_t,\yt_t) = \agd(f, \xt_0, t),
    \qquad\qquad
    \forall ~ t \geq 0,
\end{align*}
For $t \geq 1$ we can combine \cref{eq:xstep} and \cref{eq:ystep} to obtain
\begin{align*}
    x_{t+1}
    &= (1+\gamma_t) x_{t} - \gamma_{t} x_{t-1} - \eta \nabla f((1+\gamma_t) x_{t} - \gamma_{t} x_{t-1}).
\end{align*}
Using our notation for $\dltx{t}\eqdef x_t-\xt_t$ and the fact that for a quadratic $f$, differences between gradients can be expressed as $\nabla f(x) - \nabla f(x') = H (x-x')$ for any $x,x'$,
\begin{align*}
    \dltx{t+1}
    &= (1+\gamma_t) \dltx{t} - \gamma_{t} \dltx{t-1} - \eta H((1+\gamma_t) \dltx{t} - \gamma_{t} \dltx{t-1}).
\end{align*}
We can rewrite in matrix form,
\begin{align*}
\begin{pmatrix}
\dltx{t+1} \\
\dltx{t}
\end{pmatrix}
=
\begin{pmatrix}
(1+\gamma_t)(I-\eta H) & -\gamma_t (I-\eta H) \\
1 & 0
\end{pmatrix}
\begin{pmatrix}
\dltx{t} \\
\dltx{t-1}
\end{pmatrix}.
\end{align*}
Thus,
\begin{align*}
\begin{pmatrix}
\dltx{t+1} \\
\dltx{t}
\end{pmatrix} &=
\prod_{k=1}^t
\begin{pmatrix}
    (1+\gamma_k)(I-\eta H) & -\gamma_k (I-\eta H) \\
    1 & 0
    \end{pmatrix}
    \begin{pmatrix}
    \dltx{1} \\
    \dltx{0}
    \end{pmatrix}.
\end{align*}
We can bound the norm of $\dltx{1}$,
\begin{align*}
    \norm{\dltx{1}}
    = \norm{\dlty{0}-\eta \dltf{0}}
    = \norm{\dltx{0}-\eta \dltf{0}}
    \leq \norm{\dltx{0}},
\end{align*}
where the first equality comes from \cref{eq:dxstep}, the second in the initialization and the third is \cref{cor:contractive}.
Since $0 \preceq H \preceq \beta I$, $0 \preceq I-\eta H \preceq I$. Thus, using \cref{lemma:quad_multidim} and the triangle inequality we obtain
\begin{align*}
    \norm{\dltx{t+1}}
    &\leq \norm{\brk{\dltx{t+1},\dltx{t}}}
    \leq 2(t+1) \norm{(\dltx{1},\dltx{0})}
    \leq 4(t+1)\epsilon.
\end{align*}
Thus, the initialization stability after $T$ iterations is upper bounded by $4T\epsilon$.
\end{proof}

\subsection{Proof of \cref{lemma:quad_multidim}} \label{sec:quad_multidim}

First we prove the following claim.
\begin{claim} \label{clm:schur_power}
Let
$T = \begin{pmatrix}
\lambda_1 & c \\
0 & \lambda_2
\end{pmatrix}$,
then for all $t>0$,
$T^{t} = \begin{pmatrix}
\lambda_1^{t} & c \sum_{i=0}^{t-1} \lambda_1^{i} \lambda_2^{t-i-1} \\
0 & \lambda_2^{t}
\end{pmatrix}$.
\end{claim}

\begin{proof}
By induction. For $t=1$ the claim is immediate. Assuming the lemma is correct for $t-1 > 0$, then,
\begin{align*}
    T^{t} &= T^{t-1} T = 
    \begin{pmatrix}
    \lambda_1^{t-1} & c \sum_{i=0}^{t-2} \lambda_1^{i} \lambda_2^{t-i-2} \\
    0 & \lambda_2^{t-1}
    \end{pmatrix}
    \begin{pmatrix}
    \lambda_1 & c \\
    0 & \lambda_2
    \end{pmatrix} \\
    &= \begin{pmatrix}
    \lambda_1^{t} & c (\lambda_1^{t-1} + \lambda_2 \sum_{i=0}^{t-2} \lambda_1^{i} \lambda_2^{t-i-2}) \\
    0 & \lambda_2^{t}
    \end{pmatrix}
    = \begin{pmatrix}
    \lambda_1^{t} & c \sum_{i=0}^{t-1} \lambda_1^{i} \lambda_2^{t-i-1} \\
    0 & \lambda_2^{t}
    \end{pmatrix}.
\end{align*}
\end{proof}

Secondly we prove \cref{lemma:quad_multidim} in the single dimension, $A\in\R^{1 \times 1}$, restated below.
\begin{lemma} \label{lemma:quad}
Suppose 
$H_k=\begin{pmatrix}
(1+\gamma_{k})h & -\gamma_{k}h \\
1 & 0
\end{pmatrix}$,
where $0 \leq h \leq 1$ and $-1 \leq \gamma_{k} \leq 1$. Then
\begin{align*}
    \norm{\prod_{k=1}^{t}H_k}_2 \leq 2(t+1).
\end{align*}
\end{lemma}

\begin{proof}[of \cref{lemma:quad}]
Let $M_t=\prod_{k=1}^{t}H_k$. Thus,
\begin{align*}
    \norm{M_t}_2
    &= \max{\brk[c]{x^T M_t y:x,y \in \R^2 \text{ with } \norm{x}_2=\norm{y}_2=1}}.
\end{align*}
Let $\xt_t,\yt_t$ be defined as
\begin{align*}
    \xt_t,\yt_t = \argmax_{x,y \in \R^2 \text{ with } \norm{x}_2=\norm{y}_2=1}{\brk{x^T M_t y}}.
\end{align*}
For a given $h$, let $f_t$ be defined as
\begin{align*}
    f_t(\gamma_1,\dots,\gamma_t)
    &\eqdef \norm{M_t}_2
    =\xt_t^T M_t \yt_t.
\end{align*}
Note that $f$ is a multivariate linear function and as such attains it maximum in the extreme values of the variables. Thus,
\begin{align*}
    f_t(\gamma_1,\dots,\gamma_t)
    \leq \max_{\forall i\leq t, \gamma_i \in \brk[c]{-1,1}}{f_t(\gamma_1,\dots,\gamma_t)}.
\end{align*}
Using induction we will show that
\begin{align*}
    \max_{\forall i\leq t, \gamma_i \in \brk[c]{-1,1}}{f_t(\gamma_1,\dots,\gamma_t)} \leq 2(t+1).
\end{align*}
For $t=0$, $f_0=\norm{M_0}_2=\norm{I}_2=1$. For $t=1$,
\begin{align*}
    M_1 \in \brk[c]*{
    \begin{pmatrix}
    2h & -h \\
    1 & 0
    \end{pmatrix},
    \begin{pmatrix}
    0 & h \\
    1 & 0
    \end{pmatrix}},
\end{align*}
and it is easy to verify that $\norm{M_1} \leq 4$. Now for $t \geq 2$. Lets assume that $\gamma_1=-1$. Thus,
\begin{align*}
    f(\gamma_1,\dots,\gamma_t)
    &= \xt_t^T M_t \yt_t \\
    &= \xt_t^T \prod_{k=2}^{t}H_k
    \begin{pmatrix}
    0 & h \\
    1 & 0
    \end{pmatrix}
    \yt_t \\
    &\leq f_{t-1}(\gamma_2,\dots,\gamma_t) \norm*{
    \begin{pmatrix}
    0 & h \\
    1 & 0
    \end{pmatrix}
    \yt_t
    }_2^{-1}.
\end{align*}

The last transition is due to the optimal values of $\xt_{t-1}$ and $\yt_{t-1}$ in the definition of $f_{t-1}$.
Since $\norm*{
    \begin{pmatrix}
    0 & h \\
    1 & 0
    \end{pmatrix}
    \yt_t
    }_2^{-1} \leq 1$ ($0\leq h \leq 1$),
we conclude using the induction assumption.
Similarly, if $\gamma_t=-1$,
\begin{align*}
    f(\gamma_1,\dots,\gamma_t)
    &\leq \norm*{
    \xt_t^T
    \begin{pmatrix}
    0 & h \\
    1 & 0
    \end{pmatrix}
    }_2^{-1}
    f_{t-1}(\gamma_1,\dots,\gamma_{t-1}),
\end{align*}
and again we obtain our result with the induction assumption.
If $\gamma_k=-1$ for $k \in \brk[c]{2,\dots,t-1}$,
\begin{align*}
    H_{k+1}H_{k}H_{k-1}
    &=
    \begin{pmatrix}
    (1+\gamma_{k+1})h & -\gamma_{k+1}h \\
    1 & 0
    \end{pmatrix}
    \begin{pmatrix}
    0 & -h \\
    1 & 0
    \end{pmatrix}
    \begin{pmatrix}
    (1+\gamma_{k-1})h & -\gamma_{k-1}h \\
    1 & 0
    \end{pmatrix} \\
    &= h
    \begin{pmatrix}
    (1-\gamma_{k+1}\gamma_{k-1})h & -\gamma_{k+1}\gamma_{k-1}h \\
    1 & 0
    \end{pmatrix}.
\end{align*}
Since $\abs{-\gamma_{k+1}\gamma{k-1}} \leq 1$, $H_{k+1}H_{k}H_{k-1}$ is of the same form as our $H$ matrices, thus,
\begin{align*}
    f_t(\gamma_1,\dots,\gamma_t)
    &\leq f_{t-2}(\gamma_1, \dots, \gamma_{k-2},-\gamma_{k-1}\gamma_{k+1},\gamma_{k+2},\dots,\gamma_t),
\end{align*}
and we finish using the induction assumption for $t-2$.
We are left in the scenario where $\gamma_k=1$ for all $t \in [t]$. Thus,
\begin{align*}
    H_k=\begin{pmatrix}
    2h & -h \\
    1 & 0
    \end{pmatrix},
\end{align*}
which does not depend on $k$, so we will denote it with $H$ for all $k$. The Schur decomposition of $H$ of the form $QUQ^{-1}$ where $Q$ is unitary and $U$ is upper triangular, is
\begin{align*}
    H
    &= Q
    \begin{pmatrix}
    \lambda_1 & -(2h\lambda_2 - h + 1) \\
    0 & \lambda_2
    \end{pmatrix}
    Q^{-1},
\end{align*}
where $\lambda_1=h+\sqrt{h(h-1)}$ and $\lambda_2=h-\sqrt{h(h-1)}$ are the eigenvalues of $H$ and $Q=\frac{1}{\sqrt{1+ h}}
    \begin{pmatrix}
    \lambda_1 & -1 \\
    1 & \lambda_2
    \end{pmatrix}$.
Thus, taking $H$ to the power of $t$ and using \cref{clm:schur_power},
\begin{align*}
    H^t
    &= Q
    \begin{pmatrix}
    \lambda_1 & -(2h\lambda_2 - h + 1) \\
    0 & \lambda_2
    \end{pmatrix}^t
    Q^{-1} \\
    &= Q
    \begin{pmatrix}
    \lambda_1^t & c\sum_{i=0}^{t-1} \lambda_1^{i} \lambda_2^{t-i-1} \\
    0 & \lambda_2^t
    \end{pmatrix}
    Q^{-1},
\end{align*}
for $c=-(2h\lambda_2 - h + 1)$.
Returning to $f_t$,
\begin{align*}
    f_t(1,\dots,1)
    &= \norm*{\begin{pmatrix}
    \lambda_1^t & c\sum_{i=0}^{t-1} \lambda_1^{i} \lambda_2^{t-i-1} \\
    0 & \lambda_2^t
    \end{pmatrix}}_2 \\
    &\leq \norm*{\begin{pmatrix}
    \lambda_1^t & c\sum_{i=0}^{t-1} \lambda_1^{i} \lambda_2^{t-i-1} \\
    0 & \lambda_2^t
    \end{pmatrix}}_F.
\end{align*}
Since $\abs{\lambda_1}=\abs{\lambda_2}=h\leq 1$,
\begin{align*}
    f(1,\dots,1)
    &\leq \sqrt{2h^t+(\abs{c}(t-1)h^{t-1})^2}
    \leq \sqrt{2+\abs{c}^2(t-1)^2}.
\end{align*}
\begin{align*}
    \abs{c}^2
    &= \abs{h-1-2h^2+i\cdot 2h \sqrt{h(1-h)}}^2 \\
    &= h^2-2h-4h^3+1+4h^2+4h^4+4h^3-4h^4 \\
    &= 5 h^2 -2h +1
    \leq 4.
\end{align*}
Thus,
\begin{align*}
    f(1,\dots,1)
    \leq \sqrt{2+4(t-1)^2}
    \leq \sqrt{4(t+1)^2}
    \leq 2(t+1),
\end{align*}
and we conclude our lemma.
\end{proof}

Finally we move to our proof.
\begin{proof}[of \cref{sec:quad_multidim}]
Since $A$ is symmetric, it can be written as $A=Q D Q^{-1}$ where Q is an orthogonal matrix and $D$ is diagonal. Thus,
\begin{align*}
    M_k
    =
    \begin{pmatrix}
    Q & 0 \\
    0 & Q
    \end{pmatrix}
    \begin{pmatrix}
    (1+\gamma_k)D & -\gamma_k D \\
    I & 0
    \end{pmatrix}
    \begin{pmatrix}
    Q^{-1} & 0 \\
    0 & Q^{-1}
    \end{pmatrix}.
\end{align*}
Hence,
\begin{align*}
   \prod_{k=1}^t M_k
    &=
    \begin{pmatrix}
    Q & 0 \\
    0 & Q
    \end{pmatrix}
    \prod_{k=1}^t
    \begin{pmatrix}
    (1+\gamma_k)D & -\gamma_k D \\
    I & 0
    \end{pmatrix}
    \begin{pmatrix}
    Q^{-1} & 0 \\
    0 & Q^{-1}
    \end{pmatrix}.
\end{align*}
Let
\begin{align*}
    N_k=
    \begin{pmatrix}
    (1+\gamma_k)D & -\gamma_k D \\
    I & 0
    \end{pmatrix}
    \qquad\text{and}\qquad
    P_{k,i}=
    \begin{pmatrix}
    (1+\gamma_k)D_{ii} & -\gamma_k D_{ii} \\
    1 & 0
    \end{pmatrix}.
\end{align*}
We will bound $\norm{\prod_{k=1}^t N_k}_2$.
Let $\xt \in \R^{2d}$ be defined as
\begin{align*}
    \xt = \argmax_{x:\norm{x}_2=1} \norm3{\prod_{k=1}^t N_k x}_2,
\end{align*}
and let $y = \prod_{k=1}^t N_k \xt$.
Note that since $N_k$ is a block matrix where all the blocks are square and diagonal, for $i \leq d$, 
\begin{align*}
    y_i
    &= \brk*{\prod_{k=1}^t P_{k,i} (\xt_i,\xt_{i+d})^T}_1
    \qquad\text{and}\qquad
    y_{i+d}
    = \brk*{\prod_{k=1}^t P_{k,i} (\xt_i,\xt_{i+d})^T}_2.
\end{align*}
From \cref{lemma:quad}, $\norm{\prod_{k=1}^t P_{k,i}}_2\leq 2(t+1)$, hence,
\begin{align*}
    \norm{(y_i,y_{i+d})} \leq 2(t+1) \cdot \norm{(\xt_i,\xt_{i+d})}.
\end{align*}
Thus,
\begin{align*}
    \norm{y}^2
    &= \sum_{i=1}^d \norm{(y_i,y_{i+d})}^2 \\
    &\leq \sum_{i=1}^d (2(t+1))^2 \norm{(\xt_i,\xt_{i+d})}^2 \\
    &= (2(t+1))^2 \norm{\xt}^2 \\
    &= (2(t+1))^2.
\end{align*}
And we conclude with $2(t+1) \geq \norm{y}_2=\norm{\prod_{k=1}^t N_k}_2=\norm{\prod_{k=1}^t M_k}_2$,
\end{proof}

\section{Additional Uniform Stability Bounds}
\label{sec:additional-unif}

\subsection{Lower Bound for \agd, Quadratic Objectives}
\label{sec:lower_agd_quad_unif}

In this section we show that the upper bound of the uniform stability of \agd for quadratic objectives established by \citet{chen2018stability} of $\ifrac{4\eta G^2 T^2}{n}$ is tight up to a constant factor.
Our sample space is $\Z=\{1,2\}$, the loss function is
\begin{align*}
    f(w;z) = \begin{cases}
    Gw & \mbox{if } z=1; \\
    -Gw& \mbox{if } z=2.
    \end{cases}
\end{align*}
and our training samples are $S=(1,\dots,1)$ and $S'=(2,1,\dots,1)$.
Thus,
\begin{align*}
    R_S(w) = Gw
    \qquad\text{and}\qquad
    R_{S'} = \frac{n-2}{n} Gw.
\end{align*}

Let us consider two runs of the method on $R_S$ and $R_{S'}$ respectively:
\begin{align*}
    (x_t,y_t) &= \agd(R_S, x_0, t),
    \qquad
    (\xt_t,\yt_t) = \agd(R_{S'}, x_0, t),
    \qquad\qquad
    \forall ~ t \geq 0,
\end{align*}

The difference between gradients at time $t$ is
\begin{align*}
    \dltf{t} &= \nabla R_S(x_t)-\nabla R_{S'}(\xt_t) \\
    &= G-G\frac{n-2}{n}
    = \frac{2G}{n}.
\end{align*}

We will now show by induction that $\dltm{t}=\frac{t-1}{4}\cdot\frac{2G\eta}{n}$.
For $t=1$,
\begin{align*}
    \dltm{1}
    &= m_1-\mt_1 \\
    &= \gamma_1(x_1-x_0) - \gamma_1(\xt_1-\xt_0) \tag{\cref{eq:defm}} \\
    &= 0. \tag{$\gamma_1=0$}
\end{align*}

Assuming the claim for $t-1$,
\begin{align*}
    \dltm{t}
    &= \gamma_t(\dltm{t-1}-\eta\dltf{t-1}) \tag{\cref{eq:dmstep}}\\
    &= \frac{t-1}{t+2}\brk*{\frac{t-2}{4}\cdot \frac{2G\eta}{n}-\eta\dltf{t-1}} \tag{induction assumption}\\
    &= \frac{t-1}{t+2}\brk*{\frac{t-2}{4}\cdot \frac{2G\eta}{n}+\frac{2G\eta}{n}} \tag{$\dltf{t-1}=\ifrac{2G}{n}$}\\
    &= \frac{t-1}{t+2}\brk*{\frac{t+2}{4}\cdot\frac{2G\eta}{n}} \\
    &= \frac{t-1}{4} \cdot \frac{2G\eta}{n},
\end{align*}
and we finished our induction.
Now for the full dynamics,
\begin{align*}
    \dltx{t}
    &= \dlty{t-1} - \eta \dltf{t-1} \tag{\cref{eq:dxstep}} \\
    &= \dltx{t-1} + \dltm{t-1} - \eta \dltf{t-1} \tag{\cref{eq:dystep}} \\
    &= \dltx{t-1} + \frac{t-2}{4} \frac{2G\eta}{n} - \eta \frac{2G}{n} \\
    &= \dltx{t-1} + \frac{t+2}{4} \frac{2G\eta}{n}.
\end{align*}
Repeating this argument recursively,
\begin{align*}
    \dltx{T} &= \dltx{0} + \frac{2G\eta}{n} \sum_{t=1}^T \frac{t+2}{4} \\
    &= \frac{2G\eta}{n} \sum_{t=1}^T \frac{t+2}{4} \tag{$\dltx{0}=x_0-\xt_0=0$} \\
    &= \frac{2G\eta}{n} \cdot \frac{3T(T+2)}{8} \\
    &= \frac{3G\eta T(T+2)}{4n}.
\end{align*}
Thus, for both $z=1$ and $z=2$,
\begin{align*}
    \abs{f(x_T;z)-f(\xt_T;z)}
    &= G \abs{\dltx{T}} \\
    &= \frac{3}{4}\frac{G^2\eta T(T+2)}{n} \\
    &= \Theta(\ifrac{\eta G^2 T^2}{n}).
\end{align*}
Hence, the upper bound provided by \citet{chen2018stability} is tight.

\subsection{Upper Bound for \agd, Convex Objectives}

In this section we provide an exponential upper bound of uniform stability for \agd in the convex and smooth setting.

\begin{claim} \label{clm:agd_exp_upper_unif}
    Let $\ell(\cdot,z)$ be convex and $\beta$-smooth function for all $z \in \Z$. Then for \agd with step size $\eta\leq\ifrac{1}{\beta}$ and $T$ steps, for all $x_0 \in \R^d$ and $\epsilon>0$,
\begin{align*}
    \unif{\agd_T}{\ell}(n) \leq \frac{2 \eta G^2}{n} (3^{T-1}+1).
\end{align*}
\end{claim}

\begin{proof}
Let $S$, $S'$ be our sample sets which differ in only one example.
For better similarity to our arguments with initialization stability, let $f(x)=R_S(x)$ and $f'(x)=R_{S'}(x)$.
Let us consider two runs of the method initialized at $x_0$ on $f,f'$ respectively:
\begin{align*}
    (x_t,y_t) &= \agd(f, x_0, t),
    \qquad
    (\xt_t,\yt_t) = \agd(f', x_0, t),
    \qquad\qquad
    \forall ~ t \geq 0,
\end{align*}
Note now that our notation of $\dltf{t} \eqdef \nabla f(x_t)-\nabla f(\xt_t)$ does not suffice to describe the dynamics, as we now have two functions.
Using the notation $e_t \eqdef \nabla f(\xt_t) - \nabla f'(\xt_t)$, we have
\begin{align*}
    \nabla f(x_t)-\nabla f'(\xt_t)
    &= \dltf{t} + \nabla f(\xt_t) - \nabla f'(\xt_t)
    = \dltf{t} + e_t.
\end{align*}
Using the notations of $\dltx{t}=x_t-\xt_t,\dlty{t}=y_t-\yt_t,\dltm{t}=m_t-\mt_t$,
\begin{align*}
    \dltx{t} &=
    y_{t-1}-\eta \nabla f(y_{t-1}) - (\yt_{t-1}-\eta \nabla f'(\yt_{t-1})) \\
    &= \dlty{t-1} - \eta (\dltf{t-1} + e_{t-1}), \\
    \dlty{t} &=
    x_t + m_t - (\xt_t + \mt_t) \\
    &= \dltx{t} + \dltm{t}, \\
    \dltm{t}
    &= \gamma_t (m_{t-1}-\eta\nabla f(y_{t-1})) - \gamma_t (\mt_{t-1}-\eta\nabla f'(\yt_{t-1})) \\
    &= \gamma_{t}(\dltm{t-1}-\eta(\dltf{t-1} + e_{t-1})).
\end{align*}
Thus, our basic equations becomes (instead of \crefrange{eq:dxstep}{eq:dmstep})
\begin{align}
    \dltx{t} &= \dlty{t-1} - \eta (\dltf{t-1} + e_{t-1}), \label{eq:udxstep} \\
    \dlty{t} &= \dltx{t} + \dltm{t} = \dlty{t-1} - \eta (\dltf{t-1} + e_{t-1}) + \dltm{t}, \label{eq:udystep} \\
    \dltm{t} &= \gamma_{t}(\dltm{t-1}-\eta(\dltf{t-1} + e_{t-1})). \label{eq:udmstep}
\end{align}
Since $f,f'$ are different only in one term, using the Lipschitz property,
\begin{align*}
    \norm{e_t} \leq \frac{2G}{n}.
\end{align*}
We will show the bound by first proving that $\norm{\dlty{t}} \leq \frac{2 \eta G}{n} 3^{t}$ by induction.
For $t=0$ the claim is immediate. Assuming the claim is correct for $k=0,\dots,t$, we will prove for $t+1$.
\begin{align*}
    \norm{\dlty{t+1}}
    &= \norm{\dlty{t}-\eta (\dltf{t} + e_{t}) +\dltm{t+1}} \tag{\cref{eq:udystep}} \\
    &\leq \norm{\dlty{t}-\eta \dltf{t}} + \norm{\dltm{t+1}} +\eta\norm{e_{t}} \\
    &\leq \norm{\dlty{t}} + \norm{\dltm{t+1}} + \eta\norm{e_t} \tag{\cref{cor:contractive}}
    \\
    &\leq \norm{\dlty{t}} + +\eta\norm{e_t} + \norm{\dltm{1}}+\eta\sum_{k=1}^{t}(\norm{\dltf{k}}+\norm{e_{k}}) \tag{\cref{eq:udmstep} recursively with $\gamma_k \leq 1$} 
    \\
    &\leq \norm{\dlty{t}} + \eta\norm{e_t} + \eta\sum_{k=1}^{t}(\norm{\dltf{k}}+\norm{e_{k}}) \tag{$\dltm{1}=0$ since $\gamma_1=0$}
    \\
    &\leq \norm{\dlty{t}} + \frac{2\eta G}{n} + \eta\sum_{k=1}^{t}\brk*{\norm{\dltf{k}}+\frac{2 G}{n}} \tag{$\norm{e_t}\leq \ifrac{2 G}{n}$} \\
    &\leq \norm{\dlty{t}} + \eta\beta\sum_{k=1}^{t}\norm{\dlty{k}} + \frac{2\eta G}{n}(t+1) \tag{smoothness} \\
    &\leq \frac{2\eta G}{n} 3^{t} + \eta\beta\sum_{k=1}^{t}\brk*{\frac{2\eta G}{n} 3^{k}} + \frac{2\eta G}{n}(t+1) \tag{induction assumption}
    \displaybreak
    \\
    &\leq \frac{2\eta G}{n} 3^{t} + \sum_{k=1}^{t}\brk*{\frac{2\eta G}{n} 3^{k}} + \frac{2\eta G}{n}(t+1) \tag{$\eta\leq\ifrac{1}{\beta}$} \\
    &\leq \frac{2\eta G}{n} \brk*{ 3^{t} + 3\frac{3^{t}-1}{2} + t+1} \\
    &\leq \frac{2\eta G}{n} \brk*{ 3^{t} + 3\frac{3^{t}}{2} + t} \\
    &\leq \frac{2\eta G}{n} 3^{t+1} \tag{$t \geq 0 \implies 3^{t}\geq 2t$}.
\end{align*}
Hence,
\begin{align*}
    \norm{\dltx{T}}
    &= \norm{\dlty{T-1}-\eta(\dltf{T-1}+e_{T-1})} \tag{\cref{eq:udxstep}} \\
    &\leq \norm{\dlty{T-1}}+\eta\norm{e_{T-1}} \tag{\cref{cor:contractive}} \\
    &\leq \frac{2\eta G}{n} 3^{T-1}+\frac{2\eta G}{n}.
\end{align*}
Using the Lipschitz condition, we obtain our bound.
\end{proof}